%% file: MTALearning.tex
\documentclass[twoside,11pt]{article}

\usepackage{jmlr2e}

\usepackage{graphics}
\usepackage{latexsym}
\usepackage{amsmath}
\usepackage{stmaryrd}
\usepackage{ifthen}
\usepackage{array}
\usepackage{mathrsfs}
\usepackage{paralist}
\usepackage{tikz}
\usepackage{setspace}
\usetikzlibrary{arrows, automata}
\usepackage{mathrsfs}
\usepackage{multicol}

\firstpageno{1}

\input{macros}

\begin{document}

\title{Complexity of Equivalence and Learning for\\ Multiplicity Tree Automata}

\author{\name Ines Marusic \email ines.marusic@cs.ox.ac.uk \\
       \name James Worrell \email jbw@cs.ox.ac.uk \\
       \addr Department of Computer Science\\
       University of Oxford\\
       Parks Road, Oxford OX1 3QD, UK}


\maketitle

\begin{abstract}

  We consider the complexity of equivalence and learning for
  multiplicity tree automata, i.e., weighted tree automata over a
  field.  We first show that the equivalence problem is logspace
  equivalent to polynomial identity testing, the complexity of which
  is a longstanding open problem. Secondly, we derive lower bounds on
  the number of queries needed to learn multiplicity tree automata in
  Angluin's exact learning model, over both arbitrary and fixed fields.

  \citet*{habrardlearning} give an exact learning algorithm for
  multiplicity tree automata, in which the number of queries is
  proportional to the size of the target automaton and the size of a
  largest counterexample, represented as a tree, that is returned by the
  Teacher.  However, the smallest tree-counterexample may be
  exponential in the size of the target automaton.  Thus the above
  algorithm does not run in time polynomial in the size of the target
  automaton, and has query complexity exponential in the lower bound.

  Assuming a Teacher that returns minimal DAG representations of
  counterexamples, we give a new exact learning algorithm whose query
  complexity is quadratic in the target automaton size, almost
  matching the lower bound, and improving the best previously-known
  algorithm by an exponential factor.
\end{abstract}


\input{intro}

\input{prelims}

\input{tree_ac}

\input{ac_tree}

\input{Learning}

\input{lower}


\acks{The authors would like to thank Michael Benedikt for stimulating discussions and helpful advice. The first author gratefully acknowledges the support of the EPSRC.}

\bibliography{references}

\end{document}

%% file: macros.tex
\def\fed{\ifmmode
           $\lrcorner$
         \else
           {\unskip
            \nobreak
            \hfil
            \penalty50
            \hskip1em
            \null
            \nobreak
            \hfil
            $\lrcorner$
            \parfillskip=0pt
            \finalhyphendemerits=0
            \endgraf}
          \fi}

\newcommand{\N}{\mathbb{N}}

\mathchardef\ls="213C
\mathchardef\gr="213E

\newcommand{\acit}{\textbf{ACIT}}

%% file: intro.tex
\section{Introduction}
Trees are a natural model of structured data, including syntactic
structures in natural language processing, web information extraction,
and XML data on the web. Many of those applications require
representing functions from trees into the real numbers.  A broad class of
such functions can be defined by \emph{multiplicity tree automata},
which generalise \emph{probabilistic tree automata}.

Multiplicity tree automata were introduced
by~\citet*{berstel1982recognizable} under the terminology of linear
representations of tree series. They generalise classical finite tree
automata by having transitions labelled by values in a field. They
also generalise multiplicity word automata, introduced
by~{\citet*{IC::Schutzenberger1961}}, since words are a special case
of trees.  Multiplicity tree automata define many natural structural
properties of trees and can be used to model probabilistic processes
running on trees.

Multiplicity word and tree automata have been applied to a wide
variety of learning problems, including speech recognition, image
processing, character recognition, and grammatical inference; see
the paper of~\citet{BalleM12} for references.  A variety of
methods have been employed for learning such automata, including
matrix completion and spectral methods~\citep{BalleM12,DenisGH14,GybelsDH14} and principal components analysis~\citep{BaillyDR09}.

A fundamental problem concerning multiplicity tree automata is
\emph{equivalence}: given two automata, do they define the same
function on trees.  \citet*{seidlfull} proved that equivalence of
multiplicity tree automata is decidable in polynomial time assuming
unit-cost arithmetic, and in randomised polynomial time in the usual
bit-cost model. No finer analysis of the complexity of this problem
exists to date. In contrast, the complexity of equivalence for
classical nondeterministic word and tree automata has been completely
characterised: PSPACE-complete over
words~\citep{AhoHU74} and EXPTIME-complete over
trees~\citep*{seidlfull}.

Our first contribution, in Section~\ref{section:interreducibility}, is to
show that the equivalence problem for multiplicity tree automata is
logspace equivalent to polynomial identity testing, i.e., the problem
of deciding whether a polynomial given as an arithmetic circuit is
zero. The latter is known to be solvable in randomised polynomial
time~\citep*{demillo1978probabilistic,
  Schwartz:1980:FPA:322217.322225, Zippel:sparsepoly}, whereas solving
it in deterministic polynomial time is a well-studied and longstanding
open problem~\citep[see][]{AroraBarak2}.

Equivalence is closely connected to the problem of learning in the
\emph{exact learning model} of \citet*{angluin}. In this model, a
Learner actively collects information about the target function from a
Teacher through \emph{membership queries}, which ask for the value of
the function on a specific input, and \emph{equivalence queries},
which suggest a hypothesis to which the Teacher provides a
counterexample if one exists. A class of functions $\mathscr{C}$ is
\emph{exactly learnable} if there exists an exact learning algorithm such
that for any function $f \in \mathscr{C}$, the Learner identifies $f$
using polynomially many membership and equivalence queries in the size
of a shortest representation of $f$ and the size of a largest
counterexample returned by the Teacher during the execution of the algorithm.  The exact learning model is
an important theoretical model of the learning process.  It is
well-known that learnability in the exact learning model also implies
learnability in the PAC model with membership queries
\citep*{Valiant:1985:LDC:1625135.1625242}.

One of the earliest results about the exact learning model was the
proof of \citet*{Angluin87} that deterministic finite automata are
learnable.  This result was generalised
by~\citet*{DBLP:conf/dlt/DrewesH03} to show exact learnability of
deterministic finite tree automata, generalising also a result
of~\citet*{sakakibara1990learning} on the exact learnability of
context-free grammars from structural data.

Exact learnability of multiplicity automata has also been extensively
studied. \citet{five} show that multiplicity word automata can be
learned efficiently, and apply this to learn various classes of DNF
formulae and polynomials. These results were generalised
by~\citet*{klivans2006learning} to show exact learnability of
restricted algebraic branching programs and noncommutative
set-multilinear arithmetic formulae. \citet{bisht2006optimal} give an
almost tight (up to a $\mathit{log}$ factor) lower bound for the
number of queries made by any exact learning algorithm for the class
of multiplicity word automata.  Finally,~\citet*{habrardlearning} give an algorithm for learning multiplicity tree automata in the exact
model.

Our first contribution on learning multiplicity tree automata,
in Section \ref{lower_bound}, is to give lower bounds on the number of
queries needed to learn multiplicity tree automata in the exact
learning model, both for the case of an arbitrary and a fixed
underlying field. The bound in the latter case is proportional to the
automaton size for trees of a fixed maximal branching degree. To the
best of our knowledge, these are the first lower bounds on the query
complexity of exactly learning of multiplicity tree automata.

Consider a target multiplicity tree automaton whose minimal
representation $A$ has $n$ states.  The algorithm
of~\citet*{habrardlearning} makes at most $n$ equivalence queries and
number of membership queries proportional to $|A| \cdot s$, where $|A|$ is
the size of $A$ and $s$ is the size of a largest counterexample
returned by the Teacher.  Since this algorithm assumes that the
Teacher returns counterexamples represented explicitly as trees, $s$
can be exponential in $|A|$, even for a Teacher that returns
counterexamples of minimal size (see
Example~\ref{example:DAGsuccinctness}).  This observation reveals an
exponential gap between the query complexity of the algorithm of~\citet*{habrardlearning} and our above-mentioned lower bound, which is only
linear in $|A|$.  Another consequence is that the worst-case time
complexity of this algorithm is exponential in the size of the target
automaton.

Given two inequivalent multiplicity tree automata with $n$ states in
total, the algorithm of \citet{seidlfull} produces a subtree-closed
set of trees of cardinality at most $n$ that contains a tree on which
the automata differ.  It follows that the counterexample contained in
this set has at most $n$ subtrees, and hence can be represented as a
DAG with at most $n$ vertices. Thus in the context of exact learning
it is natural to consider a Teacher that can return succinctly
represented counterexamples, i.e., trees represented as DAGs.

Tree automata that run on DAG representations of finite trees were
first introduced by~\citet{charatonik1999automata} as extensions of
ordinary tree automata, and were further studied by
~\citet{anantharaman2005closure}. The automata considered
by~\citet{charatonik1999automata} and ~\citet{anantharaman2005closure}
run on fully-compressed DAGs. \citet*{fila2006automata} extended this
definition by introducing tree automata that run on DAGs that may be
partially compressed. In this paper, we employ the latter framework in
the context of learning multiplicity automata.

In Section \ref{section:thealgorithm}, we present a new exact learning
algorithm for multiplicity tree automata that achieves the same bound
in the number of equivalence queries as the algorithm of
~\citet*{habrardlearning}, while using number of membership
queries quadratic in the target automaton size and linear in the counterexample size, even when counterexamples are
given succinctly. Assuming that the Teacher provides minimal DAG
representations of counterexamples, our algorithm therefore makes
quadratically many queries in the target size. This is exponentially
fewer queries than the best previously-known
algorithm~\citep*{habrardlearning} and within a linear factor of the
above-mentioned lower bound.  Furthermore, our algorithm performs a
quadratic number of arithmetic operations in the size of the target
automaton, and can be implemented in randomised polynomial time in the
Turing model.

Like the algorithm of~\citet*{habrardlearning}, our algorithm
constructs a Hankel matrix of the target automaton.  However on
receiving a counterexample tree $z$, the former algorithm adds a new
column to the Hankel matrix for every suffix of $z$, while our
algorithm adds (at most) one new row for each subtree of $z$.
Crucially the number of suffixes may be exponential in the size of a
DAG representation of $z$, whereas the number of subtrees is only
linear in the size of a DAG representation.

An extended abstract~\citep*{MarusicW14} of this work appeared in the
proceedings of MFCS 2014. The current paper contains full proofs of
all results reported there, the formal definition of multiplicity tree
automata running on DAGs, and a refined complexity analysis of the
learning algorithm.

%% file: prelims.tex
\section{Preliminaries} \label{section:prelims}
Let $\mathbb{N}$ and~$\mathbb{N}_{0}$ denote the set of all positive and non-negative integers, respectively.  Let $n \in \mathbb{N}$. We write $[n]$ for the set $\{1, 2, \ldots, n\}$ and $I_{n}$ for the identity matrix of order~$n$.
For every $i \in [n]$, we write $e_i$ for the $i^\text{th}$ $n$-dimensional coordinate row vector.

For any matrix $A$, we write $A_{i}$ for its $i^\text{th}$ row, $A^{j}$ for its $j^\text{th}$ column, and $A_{i, j}$ for its $(i,j)^\text{th}$ entry. Given non-empty subsets $I$ and $J$ of the rows and columns of $A$, respectively, we write $A_{I, J}$ for the submatrix $(A_{i, j})_{i \in I, j \in J}$ of $A$. For singletons, we write simply $A_{i, J} := A_{ \{i \}, J}$ and $A_{I, j} := A_{I, \{j \}}$.
Let $n_1, \ldots, n_k \in \mathbb{N}$, and let $A$ be a matrix with $n_1 \cdot \ldots \cdot n_k$ rows. For every $(i_1, \ldots, i_{k}) \in  [n_1] \times\cdots \times [n_k]$, we write $A_{(i_1, \ldots, i_{k})}$ for the $(\sum_{l=1}^{k-1} (i_l -1 ) \cdot ( \prod_{ p= l+1}^{k} n_p ) + i_k)^\text{th}$ row of $A$.

Given a set $V$, we denote by $V^{*}$ the set of all finite ordered tuples of elements from $V$. For any subset $S \subseteq V$, the \emph{characteristic function} of $S$ (relative to $V$) is the function $\chi_{S}: V \to \{0, 1 \}$ such that $\chi_{S} (x) = 1$ if $x \in S$, and $\chi_{S} (x) = 0$ otherwise.

\subsection{Kronecker Product} 
Let $A$ be a matrix of dimension $m_1 \times n_1$ and $B$ a matrix of dimension $m_2 \times n_2$. The \emph{Kronecker product} of $A$ by $B$, written as $A\otimes B$, is a matrix of dimension $m_1 m_2 \times n_1 n_2$ where $(A\otimes B)_{(i_1, i_2), (j_1, j_2)} = A_{i_1 ,j_1} \cdot B_{i_2, j_2}$  for every $i_1 \in [m_1]$, $i_2 \in [m_2]$, $j_1 \in [n_1]$, $j_2 \in [n_2]$.

The Kronecker product is bilinear, associative, and has the following \emph{mixed-product property}: For any matrices $A$, $B$, $C$, $D$ such that products $A \cdot C$ and $B \cdot D$ are defined, it holds that $( A \otimes B) \cdot (C \otimes D) = (A \cdot C) \otimes (B \cdot D)$.

Let $k \in \mathbb{N}$ and $A_1, \ldots, A_k$ be matrices such that for every $l \in [k]$,  $A_l$ has $n_l$ rows. It can easily be shown using induction on $k$ that for every $(i_1, \ldots, i_k) \in [n_1]  \times \cdots \times [n_k]$, it holds that
\begin{align}
\left( A_{1} \otimes \cdots\otimes A_{k}\right)_{(i_1, \ldots, i_k)} = (A_{1})_{i_1} \otimes \cdots\otimes (A_{k})_{i_k}. \label{Kronecker_indices}
\end{align}
We write $\bigotimes_{l=1}^{k} A_{l} := A_{1} \otimes \cdots\otimes A_{k}$. 

For every $k \in \N_0$ we define the \emph{$k$-fold Kronecker power} of a matrix $A$, written as $A^{\otimes k}$, inductively by $A^{\otimes 0} = I_1$ and $A^{\otimes k} = A^{\otimes (k-1)} \otimes A$ for $k \ge 1$. 

Let $k \in \mathbb{N}_0$. For any matrices $A, B$ of appropriate dimensions, we have
\begin{align}
(A \otimes B)^{k} = A^{k} \otimes B^{k}. \label{Kronecker_mixed_kfold}
\end{align}
For any matrices $A_1, \ldots, A_k$  and $B_1, \ldots, B_k$ where product $A_l \cdot B_l$ is defined for every $l \in [k]$, we have
\begin{align}
(A_{1} \otimes \cdots\otimes A_{k}) \cdot (B_{1} \otimes \cdots\otimes B_{k}) = (A_{1} \cdot B_{1}) \otimes \cdots \otimes (A_{k} \cdot B_{k}). \label{Kronecker:kmixed}
\end{align}
Equations (\ref{Kronecker_mixed_kfold}) and (\ref{Kronecker:kmixed}) follow easily from the mixed-product property by induction on $k$.

\subsection{Finite Trees} A \emph{ranked alphabet} is a tuple  $(\Sigma, \mathit{rk})$ where $\Sigma$ is a nonempty finite set of symbols and $\mathit{rk}: \Sigma  \to \N_{0}$ is a function. Ranked alphabet $(\Sigma, \mathit{rk})$ is often written $\Sigma$ for short. For every $k \in \N_{0}$, we define the set of all \emph{$k$-ary} symbols $\Sigma_{k} := \mathit{rk}^{-1}(\{k\})$. If $\sigma \in \Sigma_{k}$ then we say that $\sigma$ has \emph{rank} (or \emph{arity}) $k$. We say that $\Sigma$ has \emph{rank} $m$ if $m = \mathit{max}\{\mathit{rk}(\sigma) : \sigma \in \Sigma \}$. 

The set of \emph{$\Sigma$-trees} (\emph{trees} for short), written as $T_{\Sigma}$, is the smallest set $T$ satisfying the following two conditions: (i) $\Sigma_{0} \subseteq T$; and  (ii) if $k \ge 1$, $\sigma \in \Sigma_{k}$, $t_{1}, \ldots, t_{k} \in T$ then $\sigma (t_{1}, \ldots, t_{k} ) \in T$. Given a $\Sigma$-tree $t$, a \emph{subtree} of $t$ is a $\Sigma$-tree consisting of a node in $t$ and all of its descendants in $t$. The set of all subtrees of $t$ is denoted by $\mathit{Sub}(t)$.

Let $\Sigma$ be a ranked alphabet and $\mathbb{F}$ be a field. A \emph{tree series} over $\Sigma$ with coefficients in $\mathbb{F}$ is a function $f: T_{\Sigma}  \to \mathbb{F}$. For every $t \in T_{\Sigma}$, we call $f (t)$ the \emph{coefficient} of $t$ in $f$. The set of all tree series  over $\Sigma$ with coefficients in $\mathbb{F}$  is denoted by $\mathbb{F}\langle\langle T_{\Sigma} \rangle\rangle$.

We define the tree series $\mathit{height}, \mathit{size}, \#_{\sigma} \in \mathbb{Q}\langle\langle
T_{\Sigma} \rangle\rangle$ where $\sigma \in
\Sigma$, as follows: (i) if $t \in \Sigma_{0}$ then $\mathit{height}(t) = 0$,
$\mathit{size}(t) = 1$, $\#_{\sigma} (t) = \chi_{\{ t =
  \sigma \}}$; and (ii) if $t = a (t_{1}, \ldots, t_{k} )$ where $k
\ge 1$, $a \in \Sigma_{k}$, $t_{1}, \ldots, t_{k} \in T_{\Sigma}$ then $\mathit{height}(t) = 1+ \mathit{max}_{i \in [k]}
\mathit{height}(t_{i})$, $\mathit{size}(t) = 1+ \sum_{i \in [k]}
\mathit{size}(t_{i})$, $\#_{\sigma}(t) = \chi_{\{ a =
  \sigma \}}+ \sum_{i \in [k]} \#_{\sigma}(t_{i})$, respectively.
For every $n \in \N_{0}$, we define the sets $T_{\Sigma}^{< n} := \{t \in T_{\Sigma} : \mathit{height}(t) < n\}$, $T_{\Sigma}^{n} := \{t \in T_{\Sigma} : \mathit{height}(t) = n\}$, and $T_{\Sigma}^{\le  n} := T_{\Sigma}^{< n} \cup T_{\Sigma}^{n}$.

Let $\Box$ be a nullary symbol not contained in $\Sigma$. The set
$C_{\Sigma}$ of \emph{$\Sigma$-contexts} (\emph{contexts} for short)
is the set of $(\{\Box\} \cup \Sigma)$-trees in which $\Box$ occurs
exactly once.  The \emph{concatenation} of $c \in C_{\Sigma}$ and $t \in T_{\Sigma} \mathop{\dot{\cup}} C_{\Sigma}$, written as $c [t]$, is the tree
obtained by substituting $t$ for $\Box$ in $c$. A \emph{suffix} of a
$\Sigma$-tree $t$ is a $\Sigma$-context $c$ such that $t = c
[t^{\prime}]$ for some $\Sigma$-tree $t^{\prime}$.  The
\emph{Hankel matrix} of a tree series $f \in \mathbb{F}\langle\langle
T_{\Sigma} \rangle\rangle$ is the matrix $H: T_{\Sigma} \times
C_{\Sigma} \to \mathbb{F}$ such that $H_{t, c} = f(c [t])$ for every
$t \in T_{\Sigma} $ and $c \in C_{\Sigma}$.

\subsection{Multiplicity Tree Automata} 
Let $\mathbb{F}$ be a field. An $\mathbb{F}$-\emph{multiplicity tree automaton} ($\mathbb{F}$-\emph{MTA}) is a quadruple $A = (n, \Sigma, \mu, \gamma)$ which consists of the \emph{dimension} $n \in \N_{0}$ representing the number of states, a ranked alphabet $\Sigma$, the \emph{tree representation} $\mu = \{\mu(\sigma) : \sigma \in \Sigma\}$ where for every symbol $\sigma \in \Sigma$, matrix $\mu(\sigma)  \in \mathbb{F}^{ n^{\mathit{rk}(\sigma)} \times n }$ represents the \emph{transition matrix} associated to $\sigma$, and
the \emph{final weight vector} $\gamma \in \mathbb{F}^{n \times 1}$. The \emph{size} of the automaton $A$, written as $|A|$, is defined as 
\begin{align*}
|A| := \sum_{\sigma \in \Sigma}
n^{\mathit{rk}(\sigma) + 1} + n.
\end{align*}
That is, the size of $A$ is the total
number of entries in all transition matrices and the final weight
vector.\footnote{We measure size assuming explicit rather than sparse representations of the transition matrices and final weight vector because minimal automata are only unique up to change of basis (see Theorem \ref{thm:minimal_unique}).}

We extend
the tree representation $\mu$ from $\Sigma$ to $T_{\Sigma}$ by defining
\begin{align*}
\mu(\sigma(t_{1}, \ldots, t_{k})) :=
(\mu(t_{1}) \otimes \cdots \otimes \mu(t_{k}) ) \cdot \mu(\sigma)
\end{align*} 
for every $\sigma \in \Sigma_{k}$ and $t_{1}, \ldots,t_{k} \in
T_{\Sigma}$. The tree series $\| A \|\in \mathbb{F}\langle\langle
T_{\Sigma} \rangle\rangle$ \emph{recognised} by $A$ is defined by ${\|
  A \| (t) = \mu(t) \cdot \gamma}$ for every $t \in
T_{\Sigma}$.  Note that a $0$-dimensional automaton necessarily recognises a zero tree series.
Two automata $A_1$, $A_2$ are said to be \emph{equivalent} if $\| A_1 \| \equiv \| A_2 \|$.

We further extend $\mu$ from $T_{\Sigma}$ to $C_{\Sigma}$ by treating $\Box$ as a unary symbol and defining $\mu (\Box) := I_{n}$. This allows to define $\mu(c) \in \mathbb{F}^{n \times n}$ for every $c = \sigma (t_{1}, \ldots, t_{k}) \in C_{\Sigma}$ inductively as $\mu(c)  :=  \left(\mu(t_{1}) \otimes \cdots \otimes \mu(t_{k}) \right) \cdot \mu(\sigma)$. It is easy to see that $\mu(c [t]) = \mu(t) \cdot \mu(c)$ for every $t \in T_{\Sigma}$ and $c \in C_{\Sigma}$.

Let $A_{1} = (n_{1},\Sigma, \mu_{1}, \gamma_{1})$ and $A_{2} =
(n_{2},\Sigma, \mu_{2}, \gamma_{2})$ be two $\mathbb{F}$-multiplicity tree automata. 
The \emph{product} of $A_{1}$ by $A_{2}$, written as
$A_{1} \times A_{2}$, is the $\mathbb{F}$-multiplicity tree automaton $(n, \Sigma, \mu, \gamma)$ where:
\begin{itemize}
\item[$\bullet$] $n = n_{1} \cdot n_{2}$;
\item[$\bullet$] If $\sigma \in \Sigma_{k}$ then $\mu (\sigma) = P_{k} \cdot (\mu_{1}(\sigma)\otimes \mu_{2}(\sigma))$ where $P_{k}$ is a permutation matrix of order $(n_{1} \cdot n_{2})^{k}$ uniquely defined (see Remark \ref{rmk:productMTA-defn} below) by
\begin{align} \label{MTACartesian_transition}
(u_{1} \otimes \cdots \otimes u_{k}) \otimes (v_{1} \otimes \cdots \otimes v_{k}) = ((u_{1} \otimes v_{1}) \otimes \cdots \otimes (u_{k} \otimes v_{k})) \cdot P_{k}
\end{align}
for all $u_{1}, \ldots, u_{k} \in \mathbb{F}^{1 \times n_{1}}$ and $v_{1}, \ldots, v_{k} \in \mathbb{F}^{1 \times n_{2}}$;
\item[$\bullet$] $\gamma = \gamma_{1} \otimes \gamma_{2}$.
\end{itemize}

\begin{remark} \label{rmk:productMTA-defn}
We argue that for every rank $k$ of a symbol in $\Sigma$, matrix $P_{k}$ is well-defined by Equation (\ref{MTACartesian_transition}). In order to do this, it suffices to show that $P_{k}$ is well-defined on a set of basis vectors of $\mathbb{F}^{1 \times n_{1}}$ and $\mathbb{F}^{1 \times n_{2}}$ and then extend linearly. To that end, let $(e_{i}^{1})_{i \in [n_{1}]}$ and $(e_{j}^{2})_{j \in [n_{2}]}$ be bases of $\mathbb{F}^{1 \times n_{1}}$ and $\mathbb{F}^{1 \times n_{2}}$, respectively. Let us define sets of vectors
\begin{align*}
E_{1} := \{(e_{i_1}^{1} \otimes \cdots \otimes e_{i_k}^{1}) \otimes (e_{j_1}^{2} \otimes \cdots \otimes e_{j_k}^{2}) : i_1, \ldots, i_k \in [n_{1}], j_1, \ldots, j_k \in [n_{2}]\}
\end{align*}
and
\begin{align*}
E_{2} := \{(e_{i_1}^{1} \otimes e_{j_1}^{2}) \otimes \cdots \otimes (e_{i_k}^{1} \otimes e_{j_k}^{2}) : i_1, \ldots, i_k \in [n_{1}], j_1, \ldots, j_k \in [n_{2}]\}.
\end{align*}
Then, $E_{1}$ and $E_{2}$ are two bases of the vector space $\mathbb{F}^{1 \times n_{1} n_{2}}$. Therefore, $P_{k}$ is well-defined as an invertible matrix mapping basis $E_{1}$ to basis $E_{2}$.
\end{remark}

Essentially the same product construction as in the proof of the first part of the following proposition is given by \citet*[Proposition 5.1]{berstel1982recognizable} in the language of linear representations of tree series rather than multiplicity tree automata. 
\begin{proposition} \label{proposition_cartesian}
Let $A_{1}$ and $A_{2}$ be $\mathbb{Q}$-multiplicity tree automata over a ranked alphabet $\Sigma$. Then, for every $t \in T_{\Sigma}$ it holds that $\| A_{1} \times A_{2}\| (t) =  \| A_{1} \|(t) \cdot \|A_{2}\|(t)$. Furthermore, automaton $A_{1} \times A_{2}$ can be computed from $A_{1}$ and $A_{2}$ in logarithmic space.
\end{proposition}
\begin{proof}
Let $A_{1} = (n_{1},\Sigma, \mu_{1}, \gamma_{1})$, $A_{2} = (n_{2},\Sigma, \mu_{2}, \gamma_{2})$, and $A_{1} \times A_{2} = (n,\Sigma, \mu, \gamma)$.
First we show that for any $t \in T_{\Sigma}$,
\begin{align}
\mu(t) = \mu_{1}(t) \otimes \mu_{2}(t). \label{cartesian_product_claim}
\end{align}
We prove that Equation (\ref{cartesian_product_claim}) holds for all $t \in T_{\Sigma}$ using induction on $height(t)$.
The base case $t = \sigma \in \Sigma_{0}$ holds immediately by definition since $P_{0} = I_{1}$.
For the induction step, let $h \in \N_{0}$ and assume that Equation (\ref{cartesian_product_claim}) holds for every $t \in T_{\Sigma}^{\le  h}$. Take any $t \in T_{\Sigma}^{h+1}$. Then $t = \sigma(t_{1}, \ldots, t_{k})$ for some $k \ge 1$, $\sigma \in \Sigma_{k}$, and $t_1, \ldots, t_k \in T_{\Sigma}^{\le  h}$. By induction hypothesis, Equation (\ref{MTACartesian_transition}), and the mixed-product property of Kronecker product we now have 
\allowdisplaybreaks \begin{align*} 
\mu(t) & = \left( \mu(t_1) \otimes \cdots \otimes \mu(t_k) \right) \cdot \mu(\sigma)\\
& = \left( \left(\mu_{1}(t_{1}) \otimes \mu_{2}(t_{1})\right) \otimes \cdots \otimes \left(\mu_{1}(t_{k}) \otimes \mu_{2}(t_{k})\right) \right) \cdot P_{k} \cdot (\mu_{1}(\sigma)\otimes \mu_{2}(\sigma))\\
& = \left(\left(\mu_{1}(t_{1}) \otimes \cdots \otimes \mu_{1}(t_{k})\right) \otimes \left(\mu_{2}(t_{1}) \otimes \cdots \otimes \mu_{2}(t_{k})\right)\right) \cdot (\mu_{1}(\sigma)\otimes \mu_{2}(\sigma))\\
& = \left(\left(\mu_{1}(t_{1}) \otimes \cdots \otimes \mu_{1}(t_{k})\right) \cdot \mu_{1}(\sigma)\right) \otimes \left(\left(\mu_{2}(t_{1}) \otimes \cdots \otimes \mu_{2}(t_{k})\right) \cdot \mu_{2}(\sigma)\right)\\
& = \mu_{1}(t) \otimes \mu_{2}(t).
\end{align*}
This completes the proof of Equation (\ref{cartesian_product_claim}) for all $t \in T_{\Sigma}$ by induction.
For every $t \in T_{\Sigma}$, we now have 
\allowdisplaybreaks\begin{align*}
\|A_{1} \times A_{2}\|(t) 
& = \mu (t) \cdot \gamma = (\mu_{1}(t) \otimes \mu_{2}(t)) \cdot (\gamma_{1} \otimes \gamma_{2})\\ 
&  = (\mu_{1}(t) \cdot \gamma_{1}) \otimes (\mu_{2}(t) \cdot \gamma_{2}) = \|A_{1}\|(t) \otimes \|A_{2}\|(t) =  \|A_{1}\|(t) \cdot \|A_{2}\|(t).
\end{align*}

We conclude by noting that automaton $A_{1} \times A_{2}$ can be
computed from $A_{1}$ and $A_{2}$ by an algorithm that maintains a
constant number of pointers, therefore requiring only logarithmic
space.
\end{proof}

A tree series $f$ is called \emph{recognisable} if it is recognised by some multiplicity tree automaton; such an automaton is called an \emph{MTA-representation} of $f$. An MTA-representation of $f$ that has the smallest dimension is called \emph{minimal}. The set of all recognisable tree series in $\mathbb{F}\langle\langle T_{\Sigma} \rangle\rangle$ is denoted by $\text{Rec}(\Sigma, \mathbb{F})$.

The following result was first shown by \citet*{bozapalidis1983rank}; an essentially equivalent result was later shown by \citet*{habrardlearning}.

\begin{theorem} [\citealp*{bozapalidis1983rank}]  \label{thm:Hankel}
Let $\Sigma$ be a ranked alphabet and $\mathbb{F}$ be a field. Let $f \in \mathbb{F}\langle\langle T_{\Sigma} \rangle\rangle$ and let $H$ be the Hankel matrix of $f$.  It holds that $f \in \text{Rec}(\Sigma, \mathbb{F})$ if and only if $H$ has finite rank over $\mathbb{F}$. In case $f\in \text{Rec}(\Sigma, \mathbb{F})$, the dimension of a minimal MTA-representation of $f$ is $\mathit{rank}(H)$ over $\mathbb{F}$. 
\end{theorem}

The following result by \citet*[Proposition 4]{bozapalidis1989representations} states that for any recognisable tree series, its minimal MTA-representation is unique up to change of basis. 

\begin{theorem}[\citealp{bozapalidis1989representations}]
\label{thm:minimal_unique}
Let $\Sigma$ be a ranked alphabet and $\mathbb{F}$ be a field. Let $f \in \text{Rec}(\Sigma, \mathbb{F})$ and let  $r$ be the rank (over $\mathbb{F}$) of the Hankel matrix of $f$. Let $A_{1} = (r, \Sigma, \mu_{1}, \gamma_{1})$ be an MTA-representation of $f$. Given an $\mathbb{F}$-multiplicity tree automaton $A_{2} = (r, \Sigma, \mu_{2}, \gamma_{2})$, it holds that $A_{2}$ recognises $f$ if and only if there exists an invertible matrix $U \in \mathbb{F}^{r \times r}$ such that $\gamma_{2} = U \cdot \gamma_{1}$ and $\mu_{2} (\sigma) =  U^{\otimes \mathit{rk}(\sigma)} \cdot \mu_{1} (\sigma) \cdot U^{-1}$ for every $\sigma \in \Sigma$.
\end{theorem}

\subsection{DAG Representations of Finite Trees} 
Let $\Sigma$ be a ranked alphabet. A \emph{DAG representation of a $\Sigma$-tree} (\emph{$\Sigma$-DAG} or \emph{DAG} for short) is a rooted directed acyclic ordered multigraph whose nodes are labelled with symbols from $\Sigma$ such that the outdegree of each node is equal to the rank of the symbol it is labelled with. Formally a $\Sigma$-DAG consists of a set of nodes $V$, for each node $v \in V$ a list of successors $\mathit{succ}(v) \in V^{*}$, and a node labelling $\lambda : V \to \Sigma$ where for each node $v \in V$ it holds that $\lambda (v) \in \Sigma_{|\mathit{succ}(v)|}$. Note that $\Sigma$-trees are a subclass of $\Sigma$-DAGs. 

Let $G$ be a $\Sigma$-DAG. The \emph{size} of $G$, denoted by $\mathit{size}(G)$,  is the number of nodes in $G$. The \emph{height} of $G$, denoted by $\mathit{height}(G)$,  is the length of a longest directed path in $G$. For any node $v$ in $G$, the \emph{sub-DAG} of $G$ \emph{rooted at}  $v$, denoted by $\displaystyle {G|}_{v}$,  is the $\Sigma$-DAG consisting of the node $v$ and all of its descendants in $G$. Clearly, if a node $v_{0}$ is the root of $G$ then $\displaystyle {G|}_{v_0} = G$. The set $\{\displaystyle {G|}_{v} : v \text{ is a node in } G \}$ of all the sub-DAGs of $G$ is denoted by $\mathit{Sub}(G)$.

For any $\Sigma$-DAG  $G$, we define its \emph{unfolding} into a $\Sigma$-tree, denoted by $\mathit{unfold}(G)$, inductively as follows: If the root of $G$ is labelled with a symbol $\sigma$ and has the list of successors $v_1, \ldots, v_k$, then 
\begin{align*}
\mathit{unfold}(G) = \sigma (\mathit{unfold}(\displaystyle {G|}_{v_{1}}), \ldots, \mathit{unfold}(\displaystyle {G|}_{v_{k}})). 
\end{align*}

It is easy to see that the following proposition holds.
\begin{proposition} \label{prelimsDAG:Subs}
If $G$ is a $\Sigma$-DAG, then $\mathit{Sub}(\mathit{unfold}(G)) = \mathit{unfold}[\mathit{Sub}(G)]$.
\end{proposition}

Because a context has exactly one occurrence of the symbol $\Box$, any
\emph{DAG representation of a $\Sigma$-context} is a $(\{\Box\} \cup
\Sigma)$-DAG that has a unique path from the root to the (unique)
$\Box$-labelled node. The \emph{concatenation} of a DAG $K$,
representing a $\Sigma$-context, and a $\Sigma$-DAG $G$, denoted by $K [G]$, is
the $\Sigma$-DAG obtained by substituting the root of $G$ for $\Box$
in $K$.

\begin{proposition} \label{prelimsDAG:concatenation} Let $K$ be a DAG
  representation of a $\Sigma$-context, and let $G$ be a
  $\Sigma$-DAG. Then, $\mathit{unfold}(K [G]) = \mathit{unfold}(K)
  [\mathit{unfold}(G)]$.
\end{proposition}
\begin{proof}
  The proof is by induction on $\mathit{height}(K)$. For the base case
  $\mathit{height}(K) = 0$, we have that $K = \Box$ and therefore
  $\mathit{unfold}(\Box [G]) = \mathit{unfold}(G) =
  \mathit{unfold}(\Box) [\mathit{unfold}(G)]$ for any $\Sigma$-DAG
  $G$.

For the induction step, let $h \in \mathbb{N}_0$ and assume that the
proposition holds if $\mathit{height}(K) \le h$. Let $K$ be a DAG
representation of a $\Sigma$-context such that $\mathit{height}(K) =
h+1$. Let the root of $K$ have label $\sigma$ and list of successors
$v_1, \ldots, v_k$. By definition, there is a unique path in $K$ going
from the root to the $\Box$-labelled node. Without loss of generality,
we can assume that the $\Box$-labelled node is a successor of
$v_1$. Take an arbitrary $\Sigma$-DAG $G$. Since
$\mathit{height}(\displaystyle {K|}_{v_{1}}) \le h$, we have by the
induction hypothesis that \allowdisplaybreaks \begin{align*}
  \mathit{unfold}(K [G]) & =\sigma (\mathit{unfold}(\displaystyle {K|}_{v_{1}} [G]), \mathit{unfold}(\displaystyle {K|}_{v_{2}}),  \ldots, \mathit{unfold}(\displaystyle {K|}_{v_{k}}))\\
  & =\sigma (\mathit{unfold}(\displaystyle {K|}_{v_{1}}) [\mathit{unfold}(G)], \mathit{unfold}(\displaystyle {K|}_{v_{2}}),  \ldots, \mathit{unfold}(\displaystyle {K|}_{v_{k}}))\\
  & =\sigma (\mathit{unfold}(\displaystyle {K|}_{v_{1}}), \mathit{unfold}(\displaystyle {K|}_{v_{2}}),  \ldots, \mathit{unfold}(\displaystyle {K|}_{v_{k}})) [\mathit{unfold}(G)]\\
  &= \mathit{unfold}(K) [\mathit{unfold}(G)].
\end{align*}
This completes the proof.
\end{proof}

\subsection{Multiplicity Tree Automata on DAGs}\label{prelims:mtaDAGs}
In this section, we introduce the notion of a multiplicity tree automaton on DAGs. To the best of our knowledge, this notion has not been studied before. 

Let $\mathbb{F}$ be a field, and $A = (n, \Sigma, \mu, \gamma)$  be an $\mathbb{F}$-multiplicity tree automaton. The computation of the automaton $A$ on a $\Sigma$-DAG $G = (V, E)$ is defined as follows: A \emph{run} of $A$ on $G$ is a mapping $\rho : \mathit{Sub}(G) \to \mathbb{F}^{n}$ such that for every node $v \in V$, if $v$ is labelled with $\sigma$ and has the list of successors $\mathit{succ}(v) = v_1, \ldots, v_k $ then 
\begin{align*}
\rho(\displaystyle {G|}_{v})  =  (\rho(\displaystyle {G|}_{v_{1}}) \otimes \cdots \otimes \rho(\displaystyle {G|}_{v_{k}}) ) \cdot \mu(\sigma).
\end{align*}
Automaton $A$ assigns to $G$ a \emph{weight} $\| A \| (G) \in \mathbb{F}$ where  $\| A \| (G) = \rho(G) \cdot \gamma$. 

In the following proposition, we show that the weight assigned by a multiplicity tree automaton to a DAG is equal to the weight assigned to its tree unfolding.
\begin{proposition} \label{prelimsDAG:prop_treeunfolding}
Let $\mathbb{F}$ be a field, and $A = (n, \Sigma, \mu, \gamma)$ be an $\mathbb{F}$-multiplicity tree automaton. For any $\Sigma$-DAG $G$, it holds that $\rho(G) = \mu(\mathit{unfold}(G))$ and $\| A \| (G) =  \| A \| (\mathit{unfold}(G) )$.
\end{proposition}
\begin{proof}
Let $V$ be the set of nodes of $G$. First we show that for every $v \in V$, 
\begin{align}
\rho(\displaystyle {G|}_{v}) = \mu(\mathit{unfold}(\displaystyle {G|}_{v})).  \label{prelimsDAG:prop_treeunfolding_eq1}
\end{align}
The proof is by induction on $\mathit{height}(\displaystyle {G|}_{v})$. For the base case, let $\mathit{height}(\displaystyle {G|}_{v}) = 0$. This implies that $\displaystyle {G|}_{v} = \sigma \in \Sigma_0$. Therefore, by definition we have that 
\begin{align*}
\rho(\displaystyle {G|}_{v}) = \mu (\sigma) = \mu(\mathit{unfold}(\sigma)) = \mu(\mathit{unfold}(\displaystyle {G|}_{v})).
\end{align*} 
For the induction step, let $h \in \N_{0}$ and assume that Equation (\ref{prelimsDAG:prop_treeunfolding_eq1}) holds for every $v \in V$ such that $\mathit{height}(\displaystyle {G|}_{v}) \le h$. Take any $v \in V$ such that $\mathit{height}(\displaystyle {G|}_{v}) = h+1$. Let the root of $\displaystyle {G|}_{v}$ be labelled with a symbol $\sigma$ and have list of successors $\mathit{succ}(v) = v_1, \ldots, v_k$. Then for every $j \in [k]$, we have that $\mathit{height}(\displaystyle {G|}_{v_j}) \le h$ and thus $\rho(\displaystyle {G|}_{v_j}) = \mu(\mathit{unfold}(\displaystyle {G|}_{v_j}))$ holds by the induction hypothesis. This implies that
\allowdisplaybreaks\begin{align*}
\rho(\displaystyle {G|}_{v}) &=  (\rho(\displaystyle {G|}_{v_{1}}) \otimes \cdots \otimes \rho(\displaystyle {G|}_{v_{k}}) ) \cdot \mu(\sigma)\\
&=  ( \mu(\mathit{unfold}(\displaystyle {G|}_{v_1})) \otimes \cdots \otimes \mu(\mathit{unfold}(\displaystyle {G|}_{v_k})) ) \cdot \mu(\sigma)\\
&= \mu(\sigma (\mathit{unfold}(\displaystyle {G|}_{v_{1}}), \ldots, \mathit{unfold}(\displaystyle {G|}_{v_{k}})))\\
&= \mu(\mathit{unfold}(\displaystyle {G|}_{v}))
\end{align*}
which completes the proof of Equation (\ref{prelimsDAG:prop_treeunfolding_eq1}) for all $v \in V$ by induction.

Taking $v$ to be the root of $G$, we get from Equation (\ref{prelimsDAG:prop_treeunfolding_eq1}) that $\rho(G) = \mu(\mathit{unfold}(G))$. Therefore, 
$\| A \| (G) = \rho(G) \cdot \gamma = \mu(\mathit{unfold}(G)) \cdot \gamma = \| A \| (\mathit{unfold}(G) )$. 
\end{proof}

\begin{example} \label{example:DAGsuccinctness}
Let $\Sigma = \{\sigma_{0}, \sigma_{2}\}$ be a ranked alphabet such that $\mathit{rk}(\sigma_{0}) = 0$ and $\mathit{rk}(\sigma_{2}) = 2$. Take any $n \in \N$. Let $t_{n}$, depicted in Figure \ref{fig:binarytree_n}, be the prefect binary $\Sigma$-tree  of height $n-1$. Note that $\mathit{size}(t_{n}) =  O (2^{n})$. Define a $\mathbb{Q}$-MTA $A = (n, \Sigma, \mu, e_{1})$ such that $\mu (\sigma_{0}) = e_n \in \mathbb{F}^{1 \times n}$ and $\mu (\sigma_{2}) \in \mathbb{F}^{n^{2} \times n}$ where $\mu(\sigma_{2})_{(i+1, i+1 ), i} = 1$ for every $i \in [n-1]$, and all other entries of $\mu (\sigma_{2})$ are zero. It is easy to see that $\| A \| (t_{n}) = 1$ and $\|A \|(t) = 0$ for every $t \in T_{\Sigma} \setminus \{t_{n}\}$. 

\begin{figure}
\centering
\begin{minipage}{.7\textwidth}
  \centering
 \begin{tikzpicture}[-, >=stealth', shorten >=0.5pt, auto, node distance=7cm, semithick]
  \tikzstyle{every node}=[circle,draw]
 \tikzstyle{every state}=[minimum size=4cm]
    \node [label= right: {$n$}] {$\sigma_{2}$} 
        child { node {$\sigma_{2}$} child { node {$\sigma_{2}$}  child{node [draw=none] {$\vdots$} child{node {$\sigma_{0}$}} child{node [draw=none] {} edge from parent[draw=none]}} child{node [draw=none] {} [dotted]}}
        child { node {$\sigma_{2}$} child{node [draw=none] {$\vdots$} [dotted] child{node [draw=none] {$\cdots$} edge from parent[draw=none]}}}}
        child{node [draw=none] {} edge from parent[draw=none]}
        child { node [label= right: {$n-1$}]{$\sigma_{2}$} child { node {$\sigma_{2}$} child{node [draw=none] {$\vdots$} [dotted] child{node [draw=none] {$\cdots$} edge from parent[draw=none]}}}
        child { node [label= right: {$n-2$}]{$\sigma_{2}$}  child{node [draw=none] {} [dotted]} child{node [draw=none] {$\vdots$} child{node [draw=none] {} edge from parent[draw=none]} child{node [label= right: {$1$}]{$\sigma_{0}$}}}}};
   \end{tikzpicture}
  \caption{Tree $t_{n}$}
  \label{fig:binarytree_n}
\end{minipage}%
\begin{minipage}{.35\textwidth}
  \centering
  \begin{tikzpicture}[>=stealth', -, shorten >=1pt, auto, node distance=1.5cm, semithick]
  \tikzstyle{every state}=[minimum size=0.5cm]

  \node [state] (A)[label= right: {$1$}] {$\sigma_{0}$};
  \node [draw=none] (B) [above of=A]  {$\vdots$};
  \node [state] (C) [above of=B] [label= right: {$n-2$}] {$\sigma_{2}$};
  \node [state] (D) [above of=C] [label= right: {$n-1$}] {$\sigma_{2}$};
  \node [state] (E) [above of=D]  [label= right: {$n$}] {$\sigma_{2}$};

  \path (A) edge [bend left]   node {} (B)
                 edge [bend right] node {} (B)
            (B) edge [bend left]   node {} (C)
                 edge [bend right] node {} (C)
            (C) edge [bend left]   node {} (D)
                 edge [bend right] node {} (D)
            (D) edge [bend left]   node {} (E)
                 edge [bend right] node {} (E);
\end{tikzpicture}
  \caption{DAG $G_{n}$}
  \label{fig:DAGbinary_n}
\end{minipage}
\end{figure}
Let $B$ be the $0$-dimensional $\mathbb{Q}$-MTA over $\Sigma$ (so that $\|B\| \equiv 0$). Suppose we were to check whether automata $A$ and $B$ are equivalent. Then the only counterexample to their equivalence, namely the tree $t_n$, has size $O (2^{n})$. Note, however, that $t_{n}$ has an exponentially more succinct DAG representation $G_{n}$, given in Figure \ref{fig:DAGbinary_n}. 
\end{example}

\subsection{Arithmetic Circuits}
An \emph{arithmetic circuit} is a finite directed acyclic vertex-labelled multigraph whose vertices, called \emph{gates}, have indegree $0$ or $2$. Vertices of indegree $0$ are called \emph{input gates} and are labelled with a constant $0$ or $1$, or a variable from the set $\{ x_{i} : i \in \mathbb{N}\}$. Vertices of indegree $2$ are called \emph{internal gates} and are labelled with an arithmetic operation $+$, $\times$, or $-$. We assume that there is a  unique gate with outdegree $0$ called the \emph{output gate}. An arithmetic circuit is called \emph{variable-free} if all input gates are labelled with $0$ or $1$. 

Given two gates $u$ and $v$ of an arithmetic circuit $C$, we call $u$ a \emph{child} of $v$ if $(u, v)$ is a directed edge in $C$. The \emph{size} of $C$ is the number of gates in $C$. The height of a gate $v$ in $C$, written as $\mathit{height} (v)$, is the length of a longest directed path from an input gate to $v$. The \emph{height} of $C$ is the maximal height of a gate in $C$. 

An arithmetic circuit $C$ computes a polynomial over the integers as follows: An input gate of $C$ labelled with $\alpha \in \{0, 1\} \cup \{ x_{i} : i \in \mathbb{N}\}$ computes the polynomial $\alpha$. An internal gate of $C$ labelled with $* \in \{+, \times, - \}$ computes the polynomial $p_1 * p_2$ where $p_1$ and $p_2$ are the polynomials computed by its children. For any gate $v$ in $C$, we write $f_v$ for the polynomial computed by $v$. The \emph{output} of $C$, written $f_{C}$, is the polynomial  computed by the output gate of $C$. The \emph{Arithmetic Circuit Identity Testing (\acit)} problem asks whether the output of a given arithmetic circuit is equal to the zero polynomial.

\subsection{The Learning Model}
In this paper we work with the \emph{exact learning model} of \citet*{angluin}: Let $f$ be a \emph{target function}. A \emph{Learner} (\emph{learning algorithm}) may, in each step, propose a hypothesis function $h$ 
by making an \emph{equivalence query} to a \emph{Teacher}. If $h$ is equivalent to $f$, then the Teacher returns YES and the Learner succeeds and halts. Otherwise, the Teacher returns NO with a \emph{counterexample}, which is an assignment $x$ such that $h(x) \neq f(x)$. Moreover, the Learner may query the Teacher for the value of the function $f$ on a particular assignment $x$ by making 
a \emph{membership query} on $x$. The Teacher returns the value $f(x)$.
 
We say that a class of functions $\mathscr{C}$ is \emph{exactly learnable} if there is a Learner that for any target function $f \in \mathscr{C}$, outputs a hypothesis $h \in \mathscr{C}$ such that $h(x) = f(x)$ for all assignments $x$, and does so in time polynomial in the size of a shortest representation of $f$ and the size of a largest counterexample returned by the Teacher. We moreover say that the class $\mathscr{C}$ is \emph{exactly learnable in (randomised) polynomial time} if the learning algorithm can be implemented to run in (randomised) polynomial time in the Turing model.

%% file: tree_ac.tex
\section{MTA Equivalence is interreducible with $\acit$} \label{section:interreducibility}

In this section, we show that the equivalence problem for
$\mathbb{Q}$-multiplicity tree automata is logspace interreducible
with $\acit$.  A related result, characterising equivalence of
probabilistic visibly pushdown automata on words in terms of
polynomial identity testing, was shown by~\citet{KieferMOWW13}.
On several occasions in this section, we will implicitly make use of
the fact that a composition of two logspace reductions is again a logspace reduction~\citep*[Lemma 4.17]{AroraBarak2}.

\subsection{From MTA Equivalence to $\acit$}
In this section, we present a logspace reduction from the equivalence problem for $\mathbb{Q}$-MTAs to $\acit$. We start with the following lemma.

\begin{lemma} \label{tree_ac_lemma}
Given an integer $n \in \N$  and a $\mathbb{Q}$-multiplicity tree automaton $A$ over a ranked alphabet $\Sigma$, one can compute, in logarithmic space in $|A|$ and $n$, a variable-free arithmetic circuit that has output $\sum_{t \in T_{\Sigma}^{<n}} \|A\| (t)$.
\end{lemma}

\begin{proof} 
Let $A= (r, \Sigma, \mu, \gamma)$, and let $m$ be the rank of $\Sigma$. By definition, it holds that
\begin{align}
\sum_{t \in T_{\Sigma}^{<n}} \|A\| (t) = \left( \sum_{t \in T_{\Sigma}^{< n}} \mu (t)\right) \cdot \gamma. \label{tree-ac:recursionbegin}
\end{align} 

\noindent We have $\sum_{t \in T_{\Sigma}^{< 1}} \mu (t)  = \sum_{\sigma \in \Sigma_{0}} \mu(\sigma)$. For every $i \in \N$, it holds that 
\begin{align*}
T_{\Sigma}^{< i+1} = \{\sigma(t_{1}, \ldots,t_{k}) : k \in \{0, \ldots,  m\}, \sigma \in \Sigma_{k}, t_{1}, \ldots,t_{k} \in T_{\Sigma}^{< i}\}
\end{align*}
and thus by bilinearity of Kronecker product,
\allowdisplaybreaks \begin{align}
\sum\limits_{t \in T_{\Sigma}^{< i+1}} \mu (t) 
&= \sum\limits_{k=0}^{m} \sum\limits_{\sigma \in \Sigma_{k}} \sum\limits_{t_{1} \in T_{\Sigma}^{< i}} \cdots \sum\limits_{t_{k} \in T_{\Sigma}^{< i}} \left( \mu(t_{1}) \otimes \cdots \otimes \mu(t_{k}) \right) \cdot \mu(\sigma) \notag\\
&= \sum\limits_{k=0}^{m} \sum\limits_{\sigma \in \Sigma_{k}} \left( \left( \sum_{t_{1} \in T_{\Sigma}^{< i}} \mu (t_{1}) \right)  \otimes \cdots \otimes \left( \sum_{t_{k} \in T_{\Sigma}^{< i}} \mu (t_{k}) \right) \right) \cdot \mu(\sigma) \notag\\
&= \sum\limits_{k=0}^{m} \left(\sum_{t \in T_{\Sigma}^{< i}} \mu (t) \right)^{\otimes k} \sum\limits_{\sigma \in \Sigma_{k}} \mu(\sigma). \label{eq:tree_ac_lemma}
\end{align}
In the following we define a variable-free arithmetic circuit $\Phi$ that has output  $\sum_{t \in T_{\Sigma}^{<n}} \|A\| (t)$. First, let us denote $G (i) := \sum_{t \in T_{\Sigma}^{<i}} \mu (t)$ for every $i \in \N$. Then by Equation (\ref{eq:tree_ac_lemma}) we have $G (i+1) =  \sum_{k=0}^{m} G (i)^{\otimes k} \cdot S(k)$
where $S(k) := \sum_{\sigma \in \Sigma_{k}} \mu(\sigma)$ for every $k \in \{0, \ldots, m\}$. In coordinate notation, for every $j\in [r]$ we have by Equation (\ref{Kronecker_indices}) that
\allowdisplaybreaks \begin{align}
G (i+1)_{j} =  \sum\limits_{k=0}^{m} \sum\limits_{(l_{1},\ldots, l_{k}) \in [r]^k} \prod_{a=1}^{k} G (i)_{l_{a}} \cdot S(k)_{(l_{1}, \ldots,l_{k}) , j}. \label{tree-ac:slpexpression}
\end{align} 

\noindent We present $\Phi$ as a straight-line program, with built-in constants
\begin{align*}
\{ \mu^{\sigma} _{(l_{1},\ldots, l_{k}), j}, \gamma_{j} : k \in \{0, \ldots, m\}, \sigma \in \Sigma_k, (l_{1}, \ldots, l_{k}) \in [r]^{k}, j \in [r] \}
\end{align*}
representing the entries of the transition matrices and the final weight vector of $A$, internal variables 
$\{ s^{k} _{(l_{1}, \ldots, l_{k}), j} : k \in \{0, \ldots, m\}, (l_{1}, \ldots,l_{k}) \in [r]^k, j \in [r] \}$ and $\{ g_{i, j} : i \in [n], j \in [r] \}$ evaluating the entries of matrices $S(k)$ and vectors $G(i)$ respectively, and the final internal variable $f$ computing the value of $\Phi$.

Formally, the straight-line program $\Phi$ is given in Table \ref{MTAreduction:slp}. Here the statements are given in indexed-sum and indexed-product notation, which can easily be expanded in terms of the corresponding binary operations. It follows from Equations (\ref{tree-ac:recursionbegin}) and (\ref{tree-ac:slpexpression}) that $\Phi$  computes $G(n) \cdot \gamma = \sum_{t \in T_{\Sigma}^{<n}} \|A\| (t)$.

\begin{table}[h]
\begin{tabular}{l}
\hline 
\multicolumn{1}{c}{\parbox{0.99\textwidth}{
\begin{enumerate}
\item For $j \in [r]$ \,do\; $\displaystyle g_{1 , j} \leftarrow \sum_{\sigma \in \Sigma_{0}} \mu^{\sigma}_{j}$
\item For $k \in \{0, \ldots, m\}$, $(l_{1}, \ldots, l_{k}) \in [r]^{k}$, $j \in [r]$ \,do\;
$\displaystyle s^{k}_{(l_{1}, \ldots,l_{k}), j} \leftarrow \sum _{\sigma \in \Sigma_{k}} \mu^{\sigma} _{(l_{1}, \ldots,l_{k}), j}$

\item  For $i = 1$ to $n-1$ \,do\;
\begin{enumerate}
\item[3.1.] For $k \in \{0, \ldots, m\}$, $(l_{1}, \ldots, l_{k}) \in [r]^{k}$, $j \in [r]$ \,do\;
\begin{align*}
h^{i, k}_{(l_{1}, \ldots, l_{k}), j} \leftarrow \prod_{a=1}^{k} g_{i, l_{a}} \cdot s^{k}_{(l_{1}, \ldots, l_{k}), j}
\end{align*}

\item[3.2.] For $j \in [r]$ \,do\;
\begin{align*}
g_{i+1,j} \leftarrow \sum_{k=0}^{m} \sum_{ (l_{1}, \ldots, l_{k}) \in [r]^k}  h^{i, k}_{(l_{1}, \ldots, l_{k}), j} 
\end{align*}
\end{enumerate}

\item For $j \in [r]$ \,do\; $\displaystyle f_{j} \leftarrow g_{n, j} \cdot \gamma_{j}$
\item $\displaystyle f \leftarrow \sum_{j \in [r]} f_{j}$.
\end{enumerate}
}} \\
\hline 
\end{tabular}   \\
\caption{Straight-line program $\Phi$}
\label{MTAreduction:slp}
\end{table}

The input gates of $\Phi$ are labelled with rational numbers. By separately encoding numerators and denominators, we can in logarithmic space reduce $\Phi$ to an arithmetic circuit where all input gates are labelled with integers. Moreover, without loss of generality we can assume that every input gate of $\Phi$  is labelled with $0$ or $1$. Any other integer label given in binary can be encoded as an arithmetic circuit.  

Recalling that a composition of two logspace reductions is again a logspace reduction, we conclude that the entire computation takes logarithmic space in $|A|$ and $n$. 
\end{proof}

Before presenting the reduction in Proposition \ref{reduction_fromMTA}, we recall the following characterisation \citep[Theorem 4.2]{seidlfull} of equivalence of two multiplicity tree automata over an arbitrary field. 

\begin{proposition}[{\citealp{seidlfull}}] \label{equivalenceMTA_decisionprocedure}
Suppose $A$ and $B$ are multiplicity tree automata of dimension $n_1$ and $n_2$, respectively, and over a ranked alphabet $\Sigma$. Then, $A$ and $B$ are equivalent if and only if  $\|A\|(t) =\|B\|(t)$ for every $t \in T_{\Sigma}^{< n_1 + n_2}$.
\end{proposition}

\begin{proposition} \label{reduction_fromMTA}
The equivalence problem for $\mathbb{Q}$-multiplicity tree automata is logspace reducible to $\acit$.
\end{proposition}

\begin{proof} 
Let $A$ and $B$ be $\mathbb{Q}$-multiplicity tree automata over a ranked alphabet $\Sigma$, and let $n$ be the sum of their dimensions. Proposition \ref{proposition_cartesian} implies that
\allowdisplaybreaks\begin{align*}
\displaystyle\sum\limits_{t \in T_{\Sigma}^{<n}}\left(\|A\|(t) - \|B\|(t)\right)^{2} & = \displaystyle\sum\limits_{t \in T_{\Sigma}^{<n}} \left(\|A\|(t)^{2} + \|B\|(t)^{2} - 2  \|A\|(t) \|B\|(t)\right)\\
& = \displaystyle\sum\limits_{t \in T_{\Sigma}^{<n}} \left(\| A \times A \|(t) + \| B \times B \|(t)  - 2 \| A \times B \|(t)\right).
\end{align*}

\noindent Thus by Proposition \ref{equivalenceMTA_decisionprocedure}, automata $A$ and $B$ are equivalent if and only if
\begin{align}
\displaystyle\sum\limits_{t \in T_{\Sigma}^{<n}} \| A \times A \|(t) + \sum\limits_{t \in T_{\Sigma}^{<n}} \| B \times B \|(t)  - 2 \sum\limits_{t \in T_{\Sigma}^{<n}} \| A \times B \|(t) = 0. \label{reduction_fromMTA_criterion}
\end{align}

We know from Proposition \ref{proposition_cartesian} that automata  $A \times A$, $B \times B$, and $A \times B$ can be computed in logarithmic space. Thus by Lemma \ref{tree_ac_lemma} one can compute, in logarithmic space in $|A|$ and $|B|$, variable-free arithmetic circuits that have outputs $\sum_{t \in T_{\Sigma}^{<n}}  \| A \times A \|(t)$, $\sum_{t \in T_{\Sigma}^{<n}} \| B \times B \|(t)$, and $\sum_{t \in T_{\Sigma}^{<n}} \| A \times B \|(t) $ respectively. Using Equation (\ref{reduction_fromMTA_criterion}), we can now easily construct a variable-free arithmetic circuit that has output $0$ if and only if $A$ and $B$ are equivalent.  
\end{proof}

%% file: ac_tree.tex
\subsection{From $\acit$ to MTA Equivalence}
We now present a converse reduction: from $\acit$ 
to the equivalence problem for $\mathbb{Q}$-MTAs. 

\citet[Proposition 2.2]{Allender_variablefree}  give a logspace reduction of the general $\acit$ problem to the special case of $\acit$ for variable-free circuits. The latter can, by representing arbitrary integers as differences of two non-negative integers, be reformulated as the problem of deciding whether two variable-free arithmetic circuits with only $+$ and $\times$-internal gates compute the same number.

\begin{proposition}
$\acit$ is logspace reducible to the equivalence problem for $\mathbb{Q}$-multiplicity tree automata.
\end{proposition}

\begin{proof}
Let $C_1$ and $C_2$ be two variable-free arithmetic circuits whose internal gates are labelled with $+$ or $\times$. By padding with extra gates, without loss of generality we can assume that in each circuit the children of a height-$i$ gate both have height $i-1$, $+$-gates have even height, $\times$-gates have odd height, and the output gate has an even height $h$. 

In the following we define two $\mathbb{Q}$-MTAs, $A_1$ and $A_2$, that are equivalent if and only if circuits $C_1$ and $C_2$ have the same output. Automata $A_1$ and $A_2$ are both defined over a ranked alphabet $\Sigma = \{\sigma_{0}, \sigma_{1}, \sigma_{2}\}$ where $\sigma_{0}$ is a nullary, $\sigma_{1}$ a unary, and $\sigma_{2}$ a binary symbol. Intuitively, automata $A_{1}$ and $A_{2}$ both recognise the common `tree-unfolding'  of  $C_1$ and $C_2$. 

We now derive $A_{1}$ from $C_1$; $A_{2}$ is analogously derived from $C_2$.
Let $\{v_{1}, \ldots, v_{r} \}$ be the set of gates of $C_1$ where $v_{r}$ is the output gate. Automaton $A_{1}$ has a state $q_{i}$ for every gate $v_{i}$ of $C_1$. 
Formally, $A_{1} = (r,\Sigma, \mu, e_r)$ where for every $i \in [r]$:
\allowdisplaybreaks\begin{itemize}
\item[$\bullet$] If $v_i$ is an input gate with label $1$ then $\mu (\sigma_{0})_{i} = 1$, otherwise $\mu (\sigma_{0})_{i} = 0$.
\item[$\bullet$] If $v_{i}$ is a $+$-gate with children $v_{j_1}$ and $v_{j_2}$ then $\mu (\sigma_{1})_{j_1, i}= \mu (\sigma_{1})_{j_2, i}= 1$ if $j_{1} \neq j_{2}$, $\mu (\sigma_{1})_{j_1, i}= 2$ if $j_{1} = j_{2}$, and $\mu (\sigma_{1})_{l, i}= 0$ for every $l \not\in \{j_{1}, j_{2}\}$. If $v_i$ is an input gate or a $\times$-gate then $\mu (\sigma_{1})^{i}= 0_{r \times 1}$. 
\item[$\bullet$] If $v_{i}$ is a $\times$-gate with children $v_{j_1}$ and $v_{j_2}$ then $\mu (\sigma_{2})_{(j_1, j_2), i} = 1$, and $\mu (\sigma_{2})_{(l_1, l_2), i} = 0$ for every $(l_1, l_2) \neq (j_1, j_2)$. If $v_i$ is an input gate or a $+$-gate then $\mu (\sigma_{2})^{i}= 0_{r^{2} \times 1}$. 
\end{itemize}

Define a sequence of trees $(t_n)_{n \in \N_0} \subseteq T_{\Sigma}$ by $t_0 = \sigma_0$, $t_{n+1} = \sigma_1 (t_{n})$ for $n$ odd, and $t_{n+1} = \sigma_2 (t_{n}, t_{n})$ for $n$ even. In the following, we show that $\|A_{1}\|(t_{h}) = f_{C_1}$. 
For every gate $v$ of $C_1$, by assumption it holds that all paths from $v$ to the output gate have equal length.  We now prove that for every $i \in [r]$, 
\allowdisplaybreaks\begin{align}
\mu (t_{h_{i}})_{i} = f_{v_i} \label{claim:ac_tree_reductionproperty}
\end{align}
where  $h_i := \mathit{height} (v_i)$. We use induction on $h_i \in \{0, \ldots, h\}$. For the base case, let $h_i = 0$. Then, $v_i$ is an input gate and thus by definition of automaton $A_1$ we have 
\begin{align*}
\mu(t_{h_{i}})_{i} = \mu(t_{0})_{i} =\mu(\sigma_0)_{i} =  f_{v_i}.
\end{align*} 
For the induction step, let $n \in [h]$ and assume that Equation (\ref{claim:ac_tree_reductionproperty}) holds for every gate $v_i$ of height less than $n$. Take an arbitrary gate $v_i$ of $C_1$ such that $h_i = n$. Let gates $v_{j_{1}}$ and $v_{j_{2}}$ be the children of $v_{i}$. Then $h_{j_{1}} = h_{j_{2}} = h_i - 1 = n -1$ by assumption. The induction hypothesis now implies that $\mu (t_{h_{i}-1})_{j_{1}} = f_{v_{j_{1}}}$ and $\mu (t_{h_{i}-1})_{j_{2}} = f_{v_{j_{2}}}$. Depending on the label of $v_{i}$, there are two possible cases as follows:
\begin{enumerate} [(i)]
\item If $v_{i}$ is a $+$-gate, then $h_i$ is even and thus by definition of $A_1$ we have
\begin{align*}
\mu (t_{h_i})_{i} &= \mu (\sigma_1(t_{h_{i} -1}))_{i} = \mu (t_{h_{i} -1})\cdot \mu (\sigma_1)^{i}\\
&= \mu (t_{h_i -1})_{j_{1}} + \mu (t_{h_i -1})_{j_{2}} = f_{v_{j_{1}}} +  f_{v_{j_{2}}} = f_{v_i}.
\end{align*}

\item If $v_{i}$ is a $\times$-gate, then $h_{i}$ is odd and thus by definition of $A_1$ and Equation (\ref{Kronecker_indices}) we have
\begin{align*}
\mu (t_{h_i})_{i} &= \mu (\sigma_2 (t_{h_i -1}, t_{h_i -1}))_{i} = \mu (t_{h_i -1})^{\otimes  2} \cdot \mu (\sigma_2)^{i}\\
&= \mu (t_{h_i -1})_{j_{1}} \cdot \mu (t_{h_i -1})_{j_{2}} = f_{v_{j_{1}}} \cdot  f_{v_{j_{2}}} = f_{v_i}.
\end{align*}
\end{enumerate}
This completes the proof of Equation (\ref{claim:ac_tree_reductionproperty}) by induction. Now for the output gate $v_r$ of $C_1$,  we get from Equation (\ref{claim:ac_tree_reductionproperty}) that $\mu (t_{h})_{r} = f_{v_r}$ since $h_r = h$.  
Therefore,
\begin{align*}
\|A_{1}\| (t_{h}) = \mu (t_{h}) \cdot e_{r} = \mu (t_{h})_{r}  = f_{v_r} = f_{C_1}. 
\end{align*}

Analogously, it holds that $\|A_{2}\|(t_{h}) = f_{C_2}$. It is moreover clear by construction that $\|A_{1}\|(t) = 0$ and $\|A_{2}\|(t) = 0$ for every $t \in T_{\Sigma} \setminus \{t_{h}\}$. Therefore, automata $A_{1}$ and $A_{2}$ are equivalent if and only if arithmetic circuits $C_1$ and $C_2$ have the same output.
\end{proof}

%% file: Learning.tex
\section{The Learning Algorithm} \label{section:thealgorithm}

In this section, we give an exact learning algorithm for multiplicity
tree automata. Our algorithm is polynomial in the size of a minimal automaton equivalent to the target and the size of a largest counterexample given as a DAG. As seen in Example \ref{example:DAGsuccinctness}, DAG counterexamples can be exponentially more succinct than tree counterexamples. Therefore, achieving a polynomial bound in the context of DAG representations is a more exacting criterion.

Over an arbitrary field $\mathbb{F}$, the algorithm can be seen as
running on a Blum-Shub-Smale machine that can write and
read field elements to and from its memory at unit cost and that can also
perform arithmetic operations and equality tests on field elements at unit
cost~\citep[see][]{AroraBarak2}. Over $\mathbb{Q}$, the algorithm can
be implemented in randomised polynomial time by representing rationals
as arithmetic circuits and using a coRP algorithm for equality testing
of such circuits~\citep[see][]{Allender_variablefree}.

This section is organised as follows: In Section \ref{subsect:algorithm} we present the algorithm. In Section \ref{subsect:correctnessalg} we prove correctness on trees, and then argue in Section \ref{subsect:DAGrepresentation} that the algorithm can be faithfully implemented using a DAG representation of trees. Finally, in Section \ref{subsect:complexity} we give a complexity analysis of the algorithm assuming the DAG representation.

\subsection{The Algorithm} \label{subsect:algorithm}

Let $f \in \text{Rec}(\Sigma, \mathbb{F})$ be the target function. The algorithm learns an MTA-representation of $f$ using its Hankel matrix $H$, which has finite rank over $\mathbb{F}$ by Theorem \ref{thm:Hankel}. At each stage, the algorithm maintains the following data: 
\allowdisplaybreaks \begin{itemize}
\item[$\bullet$] An integer $n \in \mathbb{N}$. 
\item[$\bullet$] A set of $n$ `rows' $X = \{ t_1,  \ldots, t_{n} \} \subseteq T_{\Sigma}$. 
\item[$\bullet$] A finite set of `columns' $Y \subseteq C_{\Sigma}$, where $\Box \in Y$.
\item[$\bullet$] A submatrix $H_{X, Y}$ of $H$ that has full row rank. 
\end{itemize}
These data determine a \textit{hypothesis automaton} $A$ of dimension $n$, whose states correspond to the rows of $H_{X, Y}$, with the $i^\text{th}$ row being the state reached after reading the tree $t_i$. The Learner makes an equivalence query on hypothesis $A$. In case the Teacher answers NO, the Learner receives a counterexample $z$. The Learner then parses $z$ bottom-up to find a minimal subtree of $z$ that is also a counterexample, and uses that subtree to augment the sets of rows and columns.

Formally, the algorithm $\mathsf{LMTA}$ is given in Table \ref{alg:exactlearnMTA}. Here for any $k$-ary symbol $\sigma \in \Sigma$ we define $\sigma (X, \ldots, X) := \{\sigma(t_{i_1}, \ldots, t_{i_{k}}) : (i_1, \ldots, i_{k}) \in [n]^{k}\}$.

\begin{table}[!h]
\begin{tabular}{l}
\hline 
\textbf{Algorithm $\mathsf{LMTA}$} \\
\hline
\multicolumn{1}{c}{\parbox{0.99\textwidth}{
\vspace{2mm}
\textbf{Target:} $f \in \text{Rec}(\Sigma, \mathbb{F})$, where $\Sigma$ has rank $m$ and $\mathbb{F}$ is a field
\begin{enumerate}
\item Make an equivalence query on the $0$-dimensional $\mathbb{F}$-MTA over $\Sigma$. \newline
If the answer is YES then \textbf{output} the $0$-dimensional $\mathbb{F}$-MTA over $\Sigma$ and halt. \newline
Otherwise the answer is NO and $z$ is a counterexample. Initialise: \newline $n \leftarrow 1$, $t_n \leftarrow z$, $X \leftarrow \{t_n\}$, $Y \leftarrow \{\Box\}$.

\item 
\begin{enumerate}
\item[2.1.] For every $k \in \{0, \ldots, m\}$, $\sigma \in \Sigma_k$, and $(i_1, \ldots, i_{k}) \in [n]^k$: \newline If $H_{\sigma(t_{i_1}, \ldots, t_{i_{k}}), Y}$ is not a linear combination of $H_{t_1, Y}, \ldots, H_{t_n, Y}$ then \newline $n \leftarrow n+1$, $t_n \leftarrow \sigma(t_{i_1}, \ldots, t_{i_{k}})$, $X \leftarrow X \cup \{t_n\}$.
\item[2.2.]  Define an $\mathbb{F}$-MTA $A = (n, \Sigma, \mu, \gamma)$ as follows:
\begin{itemize}
\item[$\bullet$] $\gamma = H_{X, \Box}$.
\item[$\bullet$] For every $k \in \{0,  \ldots, m\}$ and $\sigma \in \Sigma_k$:  \newline  Define matrix $\mu(\sigma) \in \mathbb{F}^{n^k \times  n}$ by the equation
\begin{align}
\mu(\sigma) \cdot H_{X, Y} = H_{\sigma (X, \ldots, X), Y}. \label{alg:transmatrix}
\end{align}
\end{itemize}
\end{enumerate}

\item  \begin{enumerate}
\item[3.1.] Make an equivalence query on $A$. \newline
If the answer is YES then \textbf{output} $A$ and halt. \newline
Otherwise the answer is NO and $z$ is a counterexample. Searching bottom-up, find a subtree $\sigma(\tau_{1}, \ldots, \tau_{k})$ of $z$ that satisfies the following two conditions:
\begin{enumerate}[(i)]
\item For every $j \in [k]$, $H_{\tau_j, Y} = \mu(\tau_j) \cdot H_{X, Y}$.
\item For some $c \in Y$, $H_{\sigma(\tau_{1}, \ldots, \tau_{k}), c} \neq  \mu(\sigma(\tau_{1}, \ldots, \tau_{k}))  \cdot H_{X, c}$.
\end{enumerate}

\item[3.2.] For every $j \in [k]$ and $(i_1, \ldots, i_{j-1}) \in [n]^{j-1}$: \newline $Y \leftarrow Y \cup \{c[ \sigma( t_{i_{1}}, \ldots, t_{i_{j-1}}, \Box, \tau_{j+1}, \ldots, \tau_{k})]\}$.
\item[3.3.] For every $j \in [k]$: \newline If $H_{\tau_j, Y}$ is not a linear combination of $H_{t_1, Y}, \ldots, H_{t_n, Y}$ then \newline $n \leftarrow n+1$, $t_n \leftarrow \tau_j$, $X \leftarrow X \cup \{t_n\}$.
\item[3.4.] Go to 2.
\end{enumerate}
\end{enumerate}
}} \\
\hline 
\end{tabular}   \\
\caption{Exact learning algorithm $\mathsf{LMTA}$ for the class of multiplicity tree automata}
\label{alg:exactlearnMTA}
\end{table}

\subsection{Correctness Proof} \label{subsect:correctnessalg}
In this section, we prove the correctness of the exact learning algorithm $\mathsf{LMTA}$. Specifically, we show that, given a target $f \in \text{Rec}(\Sigma, \mathbb{F})$, algorithm $\mathsf{LMTA}$ outputs a minimal MTA-representation of $f$ after at most $\mathit{rank}(H)$ iterations of the main loop. 

The correctness proof naturally breaks down into several lemmas.
First, we show that matrix $H_{X, Y}$  has full row rank.

\begin{lemma} \label{correctness:linindependence}
Linear independence of the set of vectors $\{ H_{t_1, Y}, \ldots, H_{t_n, Y} \}$ is an invariant of the loop consisting of Step 2 and Step 3. 
\end{lemma}

\begin{proof}
We argue inductively on the number of iterations of the loop. The base case $n=1$ clearly holds since $f(z) \neq 0$.

For the induction step, suppose that the set $\{ H_{t_1, Y}, \ldots, H_{t_n, Y} \}$ is linearly independent at the start of an iteration of the loop. If a tree $t \in T_{\Sigma}$ is added to $X$ during Step 2.1, then $H_{t, Y}$ is not a linear combination of $ H_{t_1, Y}, \ldots, H_{t_n, Y}$, and therefore $ H_{t_1, Y}, \ldots, H_{t_n, Y}, H_{t, Y}$ are linearly independent vectors. Hence, set $\{ H_{t_1, Y}, \ldots, H_{t_n, Y} \}$ is linearly independent at the start of Step 3. 

Unless the algorithm halts in Step 3.1, it proceeds to Step 3.2 where the set of columns $Y$ is increased, which clearly preserves linear independence of vectors  $H_{t_1, Y}, \ldots, H_{t_n, Y}$. If a tree $\tau_j$ is added to $X$ in Step 3.3, then $H_{\tau_j, Y}$ is not a linear combination of $H_{t_1, Y}, \ldots, H_{t_n, Y}$ which implies that the vectors $H_{t_1, Y}, \ldots, H_{t_n, Y}, H_{\tau_j, Y}$ are linearly independent. Hence, the set $\{ H_{t_1, Y}, \ldots, H_{t_n, Y} \}$ is linearly independent at the start of the next iteration of the loop. 
\end{proof}

Secondly, we show that Step 2.2 of $\mathsf{LMTA}$ can always be performed.

\begin{lemma} \label{correctness:transmatrix}
Whenever Step 2.2 starts, for every $k \in \{0,  \ldots, m\}$ and $\sigma \in \Sigma_k$ there exists a unique  matrix $\mu(\sigma) \in \mathbb{F}^{n^k \times  n}$ satisfying Equation (\ref{alg:transmatrix}).
\end{lemma} 

\begin{proof}
Take any $(i_1, \ldots, i_k ) \in [n]^{k}$. Step 2.1 ensures that $H_{\sigma(t_{i_1}, \ldots, t_{i_{k}}), Y}$ can be represented as a linear combination of vectors $H_{t_1, Y}, \ldots, H_{t_n, Y}$. This representation is unique since $H_{t_1, Y}, \ldots, H_{t_n, Y}$ are linearly independent vectors by Lemma \ref{correctness:linindependence}. Row $\mu(\sigma)_{(i_1, \ldots, i_k )} \in \mathbb{F}^{1 \times  n}$ is therefore uniquely defined by the equation $\mu(\sigma)_{(i_1, \ldots, i_k )} \cdot H_{X, Y} = H_{\sigma(t_{i_1}, \ldots, t_{i_{k}}), Y}$. 
\end{proof}

Thirdly, we show that Step 3.1 of $\mathsf{LMTA}$ can always be performed.

\begin{lemma} \label{correctness:counterexample}
Suppose that upon making an equivalence query on $A$ in Step 3.1, the Learner receives the answer NO and a counterexample $z$. Then there exists a subtree $\sigma(\tau_{1}, \ldots, \tau_{k})$ of $z$, where $k \in \{0, \ldots, m\}$, $\sigma \in \Sigma_k$, and $\tau_1, \ldots, \tau_{k} \in T_{\Sigma}$, that satisfies the following two conditions:
\begin{enumerate}[(i)]
\item For every $j \in [k]$, $H_{\tau_j, Y} = \mu(\tau_j) \cdot H_{X, Y}$.
\item For some $c \in Y$, $H_{\sigma(\tau_{1}, \ldots, \tau_{k}), c} \neq  \mu(\sigma(\tau_{1}, \ldots, \tau_{k}))  \cdot H_{X, c}$.
\end{enumerate}
\end{lemma}

\begin{proof}
Towards a contradiction, assume that there is no subtree $\sigma(\tau_{1}, \ldots, \tau_{k})$ of $z$ satisfying conditions (i) and (ii). We claim that then for every subtree $\tau$ of $z$, it holds that 
\begin{align}
H_{\tau, Y} = \mu(\tau) \cdot H_{X, Y}. \label{correctness:allcorrect}
\end{align}
In the following we prove this claim using induction on $\mathit{height}(\tau)$. The base case $\tau \in \Sigma_0$ follows immediately from Equation (\ref{alg:transmatrix}). For the induction step, let $0 \le h < \mathit{height}(z)$ and assume that Equation (\ref{correctness:allcorrect}) holds for every subtree  $\tau \in T_{\Sigma}^{\le  h}$ of $z$. Take an arbitrary subtree $\tau  \in T_{\Sigma}^{h+1}$ of $z$. Then $\tau = \sigma(\tau_{1}, \ldots, \tau_{k})$ for some $k \in [m]$, $\sigma \in \Sigma_k$, and $\tau_1, \ldots, \tau_{k}  \in T_{\Sigma}^{\le  h}$, where $\tau_1, \ldots, \tau_{k}$ are subtrees of $z$. The induction hypothesis implies that $H_{\tau_j, Y} = \mu(\tau_j) \cdot H_{X, Y}$ holds for every $j \in [k]$. Hence, $\tau$ satisfies condition (i). By assumption, no subtree of $z$ satisfies both conditions (i) and (ii). Thus $\tau$ does not satisfy condition (ii), i.e., it holds that $H_{\tau, Y} = \mu(\tau)  \cdot H_{X, Y}$. This completes the proof by induction.

Equation (\ref{correctness:allcorrect}) for $\tau = z$ gives $H_{z, Y} = \mu(z) \cdot H_{X, Y}$. Since $\Box \in Y$, this in particular implies that
\begin{align*}
f(z) = H_{z, \Box} = \mu(z) \cdot H_{X, \Box} = \mu(z) \cdot \gamma = \| A \|(z),
\end{align*}
which yields a contradiction since $z$ is a counterexample for the hypothesis $A$. 
\end{proof}

Finally, we show that the row set $X$ grows by at least $1$ in each iteration of the main loop.

\begin{lemma} \label{correctness:increaserank}
Every complete iteration of the Step 2 - 3 loop strictly increases the cardinality of $X$.
\end{lemma}

\begin{proof}
It suffices to show that in Step 3.3 at least one of the trees $\tau_1, \ldots, \tau_{k}$ is added to $X$. 
By Lemma \ref{correctness:linindependence}, at the start of Step 3.2 vectors $H_{t_1, Y}, \ldots, H_{t_n, Y}$  are linearly independent. Thus by condition (i) of Step 3.1, for every $j \in [k]$ it holds that
\begin{align}
H_{\tau_j, Y} = \mu(\tau_j) \cdot H_{X, Y} \label{correctness:uniquerepr}
\end{align}
and, moreover, Equation (\ref{correctness:uniquerepr}) is the unique representation of vector $H_{\tau_j, Y}$ as a linear combination of $H_{t_1, Y}, \ldots, H_{t_n, Y}$.
Clearly, vectors $H_{t_1, Y}, \ldots, H_{t_n, Y}$ remain linearly independent when Step 3.2 ends. 

Towards a contradiction, assume that in Step 3.3 none of the trees $\tau_1, \ldots, \tau_{k}$ were added to $X$. This means that for every $j \in [k]$, vector $H_{\tau_j, Y}$ 
can be represented as a linear combination of $H_{t_1, Y}, \ldots, H_{t_n, Y}$. The latter representation is unique, since vectors $H_{t_1, Y}, \ldots, H_{t_n, Y}$ are linearly independent, and is given by Equation (\ref{correctness:uniquerepr}).
By condition (ii) of Step 3.1 and Equations (\ref{alg:transmatrix}) and (\ref{Kronecker_indices}), we now have that
\allowdisplaybreaks\begin{align}
H_{\sigma(\tau_{1},  \ldots, \tau_{k}), c} &\neq  \mu(\sigma(\tau_{1}, \ldots, \tau_{k}))  \cdot H_{X, c} \notag \\ 
&= (\mu(\tau_{1}) \otimes \cdots \otimes \mu(\tau_{k}) ) \cdot \mu(\sigma) \cdot H_{X, c} \notag \\ &= \left(\mu(\tau_{1}) \otimes \ldots \otimes \mu(\tau_{k})\right)  \cdot H_{\sigma (X, \ldots, X), c} \notag \\ 
&= \sum_{(i_1, \ldots,  i_{k}) \in [n]^{k}} \left( \prod \limits_{j = 1}^{k} \mu(\tau_{j})_{i_{j}} \right)  \cdot H_{\sigma(t_{i_1}, \ldots, t_{i_{k}}), c}.\label{correctness:initialexpression}
\end{align}

\noindent By Step 3.2, we have that $c[ \sigma( t_{i_{1}}, \ldots, t_{i_{j-1}}, \Box, \tau_{j+1}, \ldots, \tau_{k})] \in Y$ for every $j \in [k]$ and $(i_1, \ldots, i_{j-1}) \in [n]^{j-1}$. Thus by Equation (\ref{correctness:uniquerepr}) for $j = k$, we have
\allowdisplaybreaks\begin{align}
&\sum_{(i_1, \ldots,  i_{k}) \in [n]^{k}}  \left( \prod \limits_{j = 1}^{k} \mu(\tau_{j})_{i_{j}} \right) \cdot H_{\sigma(t_{i_1}, \ldots, t_{i_{k}}), c}\notag \\ 
&= \sum_{(i_1, \ldots,  i_{k-1}) \in [n]^{k-1}} \left(  \prod \limits_{j =1}^{k-1} \mu(\tau_{j})_{i_{j}} \right) \sum_{i \in [n]} \mu(\tau_{k})_{i} \cdot H_{t_{i}, c[\sigma(t_{i_1}, \ldots, t_{i_{k-1}}, \Box)]} \notag \\
&= \sum_{(i_1, \ldots,  i_{k-1}) \in [n]^{k-1}} \left(  \prod \limits_{j =1}^{k-1} \mu(\tau_{j})_{i_{j}}\right) \cdot (\mu(\tau_k) \cdot H_{X, c[\sigma(t_{i_1}, \ldots, t_{i_{k-1}}, \Box)]}) \notag \\
&= \sum_{(i_1, \ldots,  i_{k-1}) \in [n]^{k-1}} \left(  \prod \limits_{j =1}^{k-1} \mu(\tau_{j})_{i_{j}}\right)  \cdot H_{\tau_{k}, c[\sigma(t_{i_1}, \ldots, t_{i_{k-1}}, \Box)]}. \label{correctness:inductiveproof}
\end{align}
Proceeding inductively as above and applying Equation (\ref{correctness:uniquerepr}) for every $j \in \{k-1, \ldots, 1\}$, we get that the expression of (\ref{correctness:inductiveproof}) is equal to $H_{\tau_{1}, c[\sigma(\Box, \tau_{2}, \ldots, \tau_{k})]}$. However, this contradicts Equation (\ref{correctness:initialexpression}). The result follows.
\end{proof}

Putting together Lemmas \ref{correctness:linindependence} - \ref{correctness:increaserank}, we conclude the following.

\begin{proposition} \label{appendix_correctness_propositon}
Let $\Sigma$ be a ranked alphabet and $\mathbb{F}$ be a field. Let $f \in \text{Rec}(\Sigma, \mathbb{F})$, let $H$ be the Hankel matrix of $f$, and $r$ be the rank (over $\mathbb{F}$) of $H$. On target $f$, the algorithm $\mathsf{LMTA}$ outputs a minimal MTA-representation of $f$ after at most $r$ iterations of the loop consisting of Step 2 and Step 3.
\end{proposition}

\begin{proof}
Lemmas \ref{correctness:transmatrix} and \ref{correctness:counterexample} show that every step of the algorithm $\mathsf{LMTA}$ can be performed. 

Theorem \ref{thm:Hankel} implies that $r$ is finite. From Lemma \ref{correctness:linindependence} we know that whenever Step 2 starts, $H_{X, Y}$ is a full row rank matrix and thus $n \le r$.  Lemma \ref{correctness:increaserank} implies that $n$ increases by at least $1$ in each iteration of the Step 2 - 3 loop. Therefore, the number of iterations of the loop is at most $r$. 

The proof follows by observing that $\mathsf{LMTA}$ halts only upon receiving the answer YES to an equivalence query.
\end{proof}

\subsection{Succinct Representations} \label{subsect:DAGrepresentation}
We now explain how algorithm $\mathsf{LMTA}$ can be correctly implemented using a DAG representation of trees. In particular, we assume that membership queries are made on $\Sigma$-DAGs, that the counterexamples are given as $\Sigma$-DAGs, the elements of $X$ are $\Sigma$-DAGs, and the elements of $Y$ are DAG representations of $\Sigma$-contexts, i.e., $(\{\Box\} \cup \Sigma )$-DAGs. 

As shown in Section \ref{prelims:mtaDAGs}, multiplicity tree automata can run directly on DAGs and, moreover, they assign equal weight to a DAG and to its tree unfolding. Crucially also, as explained in the proof of Theorem \ref{MTAlearning_complexityanalysis}, Step 3.1 can be run directly on a DAG representation of the counterexample, without unfolding. In particular, Step 3.1 involves multiple executions of the hypothesis automaton on trees. By Proposition \ref{prelimsDAG:prop_treeunfolding}, we can faithfully carry out this step using DAG representations of trees. Step 3.1 also involves considering all subtrees of a given counterexample. However, by Proposition \ref{prelimsDAG:Subs} this is equivalent to looking at all sub-DAGs of a DAG representation of the counterexample. At various points in the algorithm, we take $c \in Y$, $t \in X$ and compute their concatenation $c [t]$ in order to determine the corresponding entry $H_{t, c}$ of the Hankel matrix by a membership query. Proposition \ref{prelimsDAG:concatenation} implies that this can be done faithfully using DAG representations of $\Sigma$-trees and $\Sigma$-contexts. 

\subsection{Complexity Analysis} \label{subsect:complexity}
In this section we give a complexity analysis of our algorithm, and compare it to the best previously-known exact learning algorithm for multiplicity tree automata \citep*{habrardlearning} showing in particular an exponential improvement on the query complexity and the running time in the worst case.

\begin{theorem} \label{MTAlearning_complexityanalysis}
Let $f \in \text{Rec}(\Sigma, \mathbb{F})$ where $\Sigma$ has rank $m$ and $\mathbb{F}$ is a field.  Let $A$ be a minimal MTA-representation of $f$, and let $r$ be the dimension of $A$. Then, $f$ is learnable by the algorithm $\mathsf{LMTA}$, making $r+1$ equivalence queries, $|A|^{2} + |A| \cdot s$ membership queries, and $O(|A|^{2} + |A|\cdot r \cdot s)$ arithmetic operations, where $s$ denotes the size of a largest counterexample $z$, represented as a DAG, that is obtained during the execution of the algorithm.
\end{theorem}
\begin{proof}
Proposition \ref{appendix_correctness_propositon} implies that, on target $f$,  algorithm $\mathsf{LMTA}$ outputs a minimal MTA-representation of $f$ after at most $r$ iterations of the Step 2 - 3 loop, thereby making at most $r+1$ equivalence queries. 

Let $H$ be the Hankel matrix of $f$. From Lemma \ref{correctness:linindependence} we know that matrix $H_{X, Y}$ has full row rank, which implies that $|X| \le r$. As for the cardinality of the column set $Y$, at the end of Step 1 we have $|Y| = 1$. Furthermore, in each iteration of Step 3.2 the number of columns added to Y is at most 
\begin{align*}
\sum\limits_{j=1}^{k} n^{j-1} \le \sum\limits_{j=1}^{k} r^{j-1} = \frac{r^{k} -1}{r - 1} \le \frac{r^{m} -1}{r-1},
\end{align*} where $k$ and $n$ are as defined in Step 3.2. Since the number of iterations of Step 3.2 is at most $r-1$, we have $|Y| \le r^{m}$.

The number of membership queries made in Step 2 over the whole algorithm is 
\begin{align*}
\left(\sum_{\sigma \in \Sigma} |\sigma (X, \ldots, X)| + |X|\right) \cdot |Y|
\end{align*}
because the Learner needs to ask for the values of the entries of matrices $H_{X, Y}$ and $H_{\sigma (X, \ldots, X), Y}$ for every $\sigma \in \Sigma$.

To analyse the number of membership queries made in Step 3, we now detail the procedure by which an appropriate sub-DAG of the counterexample $z$ is found in Step 3.1. By Lemma \ref{correctness:counterexample}, there exists a sub-DAG $\tau$ of $z$ such that $H_{\tau, Y} \neq \mu(\tau) \cdot H_{X, Y}$. Thus given a counterexample $z$ in Step 3.1, the procedure for finding a required sub-DAG of $z$ is as follows: Check  if $H_{\tau, Y} = \mu(\tau) \cdot H_{X, Y}$ for every sub-DAG $\tau$ of $z$ in a nondecreasing order of height; stop when a sub-DAG $\tau$ is found such that $H_{\tau, Y} \neq \mu(\tau) \cdot H_{X, Y}$.  

In each iteration of Step 3, the Learner makes $\mathit{size}(z) \cdot |Y| \le s \cdot |Y|$ membership queries because, for every sub-DAG $\tau$ of $z$, the Learner needs to ask for the values of the entries of vector $H_{\tau, Y}$. All together, the number of membership queries  made during the execution of the algorithm is at most  
\allowdisplaybreaks\begin{align*}
&\left( \sum\limits_{\sigma \in \Sigma} |\sigma (X, \ldots, X)| + |X| \right) \cdot |Y| + (r-1) \cdot  s \cdot |Y|\\
&\le \left( \sum\limits_{\sigma \in \Sigma} r^{\mathit{rk}(\sigma)} + r \right) \cdot  r^{m} + (r-1) \cdot  s \cdot  r^{m} \le |A|^{2} + |A| \cdot s.
\end{align*}

As for the arithmetic complexity, in Step 2.1  one can determine if a vector $H_{\sigma(t_{i_1}, \ldots, t_{i_{k}}), Y}$ is a linear combination of $H_{t_1, Y}, \ldots, H_{t_n, Y}$  via Gaussian elimination using $O(n^{2} \cdot |Y|)$ arithmetic operations \citep[see][Section 2.3]{Coh93}. Analogously, in Step 3.3 one can determine if $H_{\tau_j, Y}$ is a linear combination of $H_{t_1, Y}, \ldots, H_{t_n, Y}$  via Gaussian elimination using $O(n^{2} \cdot |Y|)$ arithmetic operations. Since $|X| \le r$ and $|Y| \le r^{m}$, all together Step 2.1 and Step 3.3 require at most $O(|A|^{2})$ arithmetic operations.

Lemma \ref{correctness:transmatrix} implies that in each iteration of Step 2.2, for any $\sigma \in \Sigma$ there exists a unique matrix $\mu(\sigma) \in \mathbb{F}^{n^{\mathit{rk}(\sigma)} \times  n}$ that satisfies Equation (\ref{alg:transmatrix}). To perform an iteration of Step 2.2, we first put matrix $H_{X, Y}$ in echelon form and then, for each $\sigma \in \Sigma$, solve  Equation (\ref{alg:transmatrix}) for $\mu (\sigma)$ by back substitution. It follows from standard complexity bounds on the conversion of matrices to echelon form \citep[Section 2.3]{Coh93} that the total operation count for Step 2.2 can be bounded above by $O(|A|^{2})$.

Finally, we consider the arithmetic complexity of Step 3.1. In every iteration, for each sub-DAG $\tau$ of the counterexample $z$ the Learner needs to compute the vector $\mu(\tau)$ and the product $\mu(\tau) \cdot H_{X, Y}$. Note that $\mu(\tau)$ can be computed bottom-up from the sub-DAGs of $\tau$. Since $z$ has at most $s$ sub-DAGs, Step 3.1 requires at most $O(|A|\cdot r \cdot s)$ arithmetic operations. All together, the algorithm requires at most $O(|A|^{2} + |A|\cdot r \cdot s)$ arithmetic operations.
\end{proof}

Algorithm $\mathsf{LMTA}$ can be used to show that over $\mathbb{Q}$,
multiplicity tree automata are exactly learnable in randomised
polynomial time. The key idea is to represent numbers as arithmetic
circuits. In executing $\mathsf{LMTA}$, the Learner need only perform
arithmetic operations on circuits (addition, subtraction,
multiplication, and division), which can be done in constant time, and
equality testing, which can be done in coRP
~\citep[see][]{AroraBarak2}. These suffice for all the operations
detailed in the proof of Theorem \ref{MTAlearning_complexityanalysis};
in particular they suffice for Gaussian elimination, which can be used
to implement the linear independence checks in $\mathsf{LMTA}$.

The complexity of algorithm $\mathsf{LMTA}$ should be compared to the complexity of the algorithm of \citet{habrardlearning}, which learns multiplicity tree automata by making $r+1$ equivalence queries, $|A| \cdot s$ membership queries, and a number of arithmetic operations polynomial in $|A|$ and $s$, where $s$ is the size of the largest counterexample given as a tree. Note that the algorithm of \citet{habrardlearning} cannot be straightforwardly adapted to work directly with DAG representations of trees since when given a counterexample $z$, every suffix of $z$ is added to the set of columns. However, the tree unfolding of a  DAG can have exponentially many different suffixes in the size of the DAG. For example, the DAG in Figure \ref{fig:DAGbinary_n} has size $n$, and its tree unfolding, shown in Figure \ref{fig:binarytree_n}, has $O (2^{n})$ different suffixes.

%% file: lower.tex
\section{Lower Bounds on Query Complexity of Learning MTA} \label{lower_bound}
In this section, we study lower bounds on the query complexity of learning multiplicity tree automata in the exact learning model.
Our results generalise the corresponding lower bounds for learning multiplicity word automata by \citet{bisht2006optimal}, and make no 
assumption about the computational model of the learning algorithm.

First, we give a lower bound on the query complexity of learning
multiplicity tree automata over an arbitrary field, which is the situation of our algorithm in Section \ref{section:thealgorithm}.

\begin{theorem} 
\label{anyfield}
Any exact learning algorithm that learns the class of multiplicity tree automata of dimension at most $r$, over a ranked alphabet $(\Sigma, \mathit{rk})$ and any field, must make at least $\sum_{\sigma \in \Sigma} r^{\mathit{rk}(\sigma) + 1} - r^{2}$ queries.
\end{theorem}
\begin{proof}
Take an arbitrary exact learning algorithm $\mathsf{L}$ that learns the class of multiplicity tree automata of dimension at most $r$, over a ranked alphabet $(\Sigma, \mathit{rk})$ and over any field. 
Let $\mathbb{F}$ be any field. 

Let $\mathbb{K} := \mathbb{F}(\{z_{i, j}^{\sigma} : \sigma \in \Sigma, i \in [r^{\mathit{rk}(\sigma)}], j \in [r] \})$ be an extension field of $\mathbb{F}$, where the set $\{z_{i, j}^{\sigma} : \sigma \in \Sigma, i \in [r^{\mathit{rk}(\sigma)}], j \in [r] \}$ is algebraically independent over $\mathbb{F}$. Let us define a `generic' $\mathbb{K}$-multiplicity tree automaton $A := ( r, \Sigma, \mu, \gamma)$ where $\gamma = e_1 \in \mathbb{F}^{r \times 1}$ and $\mu(\sigma) = [ z_{i, j}^{\sigma} ]_{i, j} \in \mathbb{K}^{r^{rk(\sigma)} \times r}$ for every $\sigma \in \Sigma$. Define a tree series $f := \| A \|$. Since every  $\mathbb{F}$-multiplicity tree automaton can be obtained from $A$ by substituting values from $\mathbb{F}$ for the variables $z_{i, j}^{\sigma}$, if the Hankel matrix of $f$ had rank less than $r$ then every $r$-dimensional $\mathbb{F}$-multiplicity tree automaton would have Hankel matrix of rank less than $r$. Therefore, the Hankel matrix of $f$ has rank $r$.  

We run algorithm $\mathsf{L}$ on the target function $f$. By assumption, the output of $\mathsf{L}$ is an MTA $A^{\prime} = ( r,\Sigma, \mu^{\prime}, \gamma^{\prime})$ such that $\| A^{\prime} \| \equiv f$. Let $n$ be the number of queries made by $\mathsf{L}$ on target $f$. Let $t_1, \ldots, t_n \in T_{\Sigma}$ be the trees on which $\mathsf{L}$ either made a membership query, or which were received as counterexample to an equivalence query. Then for every $l \in [n]$, there exists a multivariate polynomial $p_l \in \mathbb{F} [(z_{i, j}^{\sigma})_{i, j, \sigma}]$ such that $f (t_l) = p_l$. 

Note that $A$ and $A^{\prime}$ are both minimal MTA-representations of $f$. Thus by Theorem \ref{thm:minimal_unique}, there exists an invertible matrix $U \in \mathbb{K}^{r \times r}$ such that $\gamma = U \cdot  \gamma^{\prime}$ and $\mu (\sigma) =  U^{\otimes \mathit{rk} (\sigma)} \cdot \mu^{\prime} (\sigma) \cdot U^{-1}$ for every $\sigma \in \Sigma$. This implies that the entries of matrices $\mu (\sigma)$, $\sigma \in \Sigma$, lie in an extension of $\mathbb{F}$ generated by the entries of $U$ and $\{p_l : l \in [n] \}$. Since the entries of matrices $\mu (\sigma)$, $\sigma \in \Sigma$, form an algebraically independent set over $\mathbb{F}$, their number is at most $r^{2} + n$. 
\end{proof}

One may wonder whether a learning algorithm could do better over a fixed field $\mathbb{F}$ by exploiting particular features of the field. In this setting, we have the following lower bound. 

\begin{theorem}\label{specificfield}
Let $\mathbb{F}$ be an arbitrary field. Any exact learning algorithm that learns the class of $\mathbb{F}$-multiplicity tree automata of dimension at most $r$, over a ranked alphabet $(\Sigma, \mathit{rk})$ with at least one unary symbol and rank $m$, must make number of queries at least 
\begin{align*}
\frac{1}{2^{m +1}} \cdot \left(\sum_{\sigma \in \Sigma} r^{\mathit{rk}(\sigma) + 1} - r^{2} - r \right). 
\end{align*} 
\end{theorem}
\begin{proof}
Without loss of generality, we can assume that $r$ is even and define $n := r/2$. Let $\mathsf{L}$ be an exact 
learning algorithm for the class of 
$\mathbb{F}$-multiplicity tree automata of dimension at most $r$, over a ranked alphabet $(\Sigma, \mathit{rk})$ with $\mathit{rk}^{-1}(\{1\}) \neq \emptyset$. We will identify a class of functions $\mathcal{C}$ such that $\mathsf{L}$ has to make at least  $\sum_{\sigma \in \Sigma} n^{\mathit{rk}(\sigma) + 1} - n^{2} - n$ queries to distinguish between the members of $\mathcal{C}$.

Let $\sigma_0, \sigma_1 \in \Sigma$ be nullary and unary symbols respectively. Let $P \in \mathbb{F}^{n \times n}$ be the permutation matrix corresponding to the cycle $(1, 2, \ldots, n)$. Define $\mathcal{A}$ to be the set of all $\mathbb{F}$-multiplicity tree automata $(2n,\Sigma, \mu, \gamma)$ where
\allowdisplaybreaks \begin{itemize}
\item[$\bullet$] $\mu(\sigma_0) =  \begin{bmatrix}
 1 ~&  0
 \end{bmatrix} \otimes e_{1}$ and $\mu(\sigma_1) = I_{2} \otimes P$;
\item[$\bullet$] For each $k$-ary symbol $\sigma \in \Sigma \setminus \{\sigma_0, \sigma_1\}$, there exists $B(\sigma) \in \mathbb{F}^{n^{k} \times n}$ such that
\begin{eqnarray*}
\mu(\sigma) =
 \begin{bmatrix}
 1 ~&  1
 \end{bmatrix} \otimes \left(  \begin{bmatrix}
 I_{n} \\
-I_{n}
 \end{bmatrix}^{\otimes k} \cdot B(\sigma)\right);
\end{eqnarray*}
\item[$\bullet$] $\gamma = \begin{bmatrix}
 1 ~&  0
 \end{bmatrix}^{\top} \otimes e_{1}^{\top}$.
\end{itemize}
We define a set of recognisable tree series $\mathcal{C} := \{\| A \| : A \in  \mathcal{A}\}$.  

In Lemma \ref{lower:propertiesofclass} we state some properties of the functions in $\mathcal{C}$. More precisely, we show that the coefficient of a tree $t \in T_{\Sigma}$ in any series $f \in \mathcal{C}$ fundamentally depends on whether $t$ has $0$, $1$, or at least $2$ nodes whose label is not $\sigma_{0}$ or $\sigma_{1}$. Here for every $i \in \N_0$ and $t \in T_{\Sigma}$, we use $\sigma_{1}^{i} (t)$ to denote the tree 
$\underbrace{\sigma_{1}( \sigma_{1}( \ldots \sigma_{1}(}_{i} t) \ldots ) )$.

\begin{lemma} \label{lower:propertiesofclass}
The following properties hold for every $f \in \mathcal{C}$ and $t \in T_{\Sigma}$:
\begin{itemize}
\item[\textit{(i)}] If $t = \sigma_{1}^{j} (\sigma_{0})$ where $j \in \{0, 1, \ldots, n-1\}$, then $f(\sigma_{0})= 1$ and $f(\sigma_{1}^{j} (\sigma_{0}))= 0$ for $j > 0$.
\item[\textit{(ii)}] If $t = \sigma_{1}^{j} (\sigma(\sigma_{1}^{i_1} (\sigma_{0}), \ldots,\sigma_{1}^{i_k} (\sigma_{0})))$ where $k \in \{0, 1, ..., m\}$, $\sigma \in \Sigma_k \setminus \{\sigma_0, \sigma_1\}$, and $j, i_1, \ldots, i_k \in \{0, 1, \ldots, n-1\}$, then $f(t) = B(\sigma)_{(1 + i_1, \ldots, 1 + i_k), (1+n-j) \bmod n}$.
\item[\textit{(iii)}] If $\sum_{\sigma \in \Sigma\setminus \{\sigma_0, \sigma_1\}} \#_{\sigma}(t) \ge 2$, then $f (t) = 0$.
\end{itemize}
\end{lemma}
\begin{proof}
Let $A =(2n,\Sigma, \mu, \gamma) \in  \mathcal{A}$ be such that $\| A \| \equiv f$. 
First, we prove property \textit{(i)}. Using Equation (\ref{Kronecker_mixed_kfold}) and the mixed-product property of Kronecker product, we get that
\begin{eqnarray}\label{zerosym}
\mu(\sigma_{1}^{j} (\sigma_{0})) &= \mu(\sigma_0) \cdot \mu(\sigma_1)^{j} = ( \begin{bmatrix}
 1 ~&  0
 \end{bmatrix} \otimes e_{1} ) \cdot (I_{2} \otimes P^{j}) = \begin{bmatrix}
 1 ~&  0
 \end{bmatrix} \otimes e_{1}  P^{j}
\end{eqnarray}
and therefore
\begin{align}
f(\sigma_{1}^{j} (\sigma_{0})) &=  \mu(\sigma_{1}^{j} (\sigma_{0})) \cdot \gamma = ( \begin{bmatrix}
 1 ~&  0
 \end{bmatrix} \otimes e_{1}  P^{j}) \cdot ( \begin{bmatrix}
 1 ~&  0
 \end{bmatrix}^{\top}  \otimes e_{1}^{\top}  )\notag\\ &=  (\begin{bmatrix}
 1 ~&  0
 \end{bmatrix} \cdot  \begin{bmatrix}
 1 ~&  0
 \end{bmatrix}^{\top}) \otimes (e_{1}  P^{j} \cdot e_{1}^{\top}) = e_{j+1} \cdot e_{1}^{\top}. \label{lower:property1}
\end{align}
If $j=0$ then the expression of (\ref{lower:property1}) is equal to $1$, otherwise the expression of (\ref{lower:property1}) is equal to $0$. This completes the proof of property \textit{(i)}.

Next, we prove property \textit{(ii)}. By the mixed-product property of Kronecker product and Equations (\ref{Kronecker_mixed_kfold}), (\ref{Kronecker:kmixed}), and (\ref{zerosym}),  we have
\allowdisplaybreaks\begin{align}
&\mu(\sigma_{1}^{j} (\sigma(\sigma_{1}^{i_1} (\sigma_{0}), \ldots,\sigma_{1}^{i_k} (\sigma_{0})))) \notag\\ 
&=  \left( \bigotimes\limits_{l=1}^{k} \mu(\sigma_{1}^{i_l} (\sigma_{0})) \right) \cdot \mu(\sigma) \cdot \mu(\sigma_1)^{j} \notag\\
&= \left(  \begin{bmatrix}
 1 
 \end{bmatrix} \otimes \bigotimes\limits_{l=1}^{k} \mu(\sigma_{1}^{i_l} (\sigma_{0})) \right) \cdot \left(  \begin{bmatrix}
 1 ~&  1
 \end{bmatrix} \otimes \left(  \begin{bmatrix}
 I_{n} \\
-I_{n}
 \end{bmatrix}^{\otimes k} \cdot B(\sigma)\right)\right) \cdot (I_{2} \otimes P)^{j}  \notag\\
&=  \left( \left(  \begin{bmatrix}
 1 
 \end{bmatrix} \cdot \begin{bmatrix}
 1 ~&  1
 \end{bmatrix}\right) \otimes\left( \bigotimes\limits_{l=1}^{k} \mu(\sigma_{1}^{i_l} (\sigma_{0}))  \cdot   \begin{bmatrix}
 I_{n} \\
-I_{n}
 \end{bmatrix}^{\otimes k} \cdot B(\sigma)\right)\right) \cdot (I_{2} \otimes P^{j}) \notag\\
&=  \left( \begin{bmatrix}
 1 ~&  1
 \end{bmatrix} \otimes\left( \bigotimes\limits_{l=1}^{k} \left( \left(\begin{bmatrix}
 1 ~&  0
 \end{bmatrix} \otimes e_{1}  P^{i_l}\right)  \cdot   \begin{bmatrix}
 I_{n} \\
-I_{n}
 \end{bmatrix}\right) \cdot B(\sigma)\right)\right) \cdot (I_{2} \otimes P^{j}) \notag \\
&=  \left( \begin{bmatrix}
 1 ~&  1
 \end{bmatrix} \otimes\left( \bigotimes\limits_{l=1}^{k} e_{1}  P^{i_l} \cdot B(\sigma)\right)\right) \cdot (I_{2} \otimes P^{j}) \notag\\
&= \left(\begin{bmatrix}
 1 ~&  1
 \end{bmatrix} \cdot I_{2}\right) \otimes  \left( \bigotimes\limits_{l=1}^{k} e_{1 + i_{l}} \cdot B(\sigma) \cdot  P^{j}\right) \notag\\
&= \begin{bmatrix}
 1 ~&  1
 \end{bmatrix} \otimes  (  B(\sigma)_{(1 + i_1, \ldots, 1 + i_k)} \cdot  P^{j}) \label{lower:propertiesofclass:part2}
\end{align}
and therefore, using the fact that $P^{n} = I_{n}$, we get that
\allowdisplaybreaks\begin{align*}
f(\sigma_{1}^{j} (\sigma(\sigma_{1}^{i_1} (\sigma_{0}), \ldots,\sigma_{1}^{i_k} (\sigma_{0})))) 
&= \mu(\sigma_{1}^{j} (\sigma(\sigma_{1}^{i_1} (\sigma_{0}), \ldots,\sigma_{1}^{i_k} (\sigma_{0})))) \cdot \gamma\\
 &= (\begin{bmatrix}
 1 ~&  1
 \end{bmatrix} \otimes  (  B(\sigma)_{(1 + i_1, \ldots, 1 + i_k)} \cdot  P^{j})) \cdot (\begin{bmatrix}
 1 ~&  0
 \end{bmatrix}^{\top} \otimes e_{1}^{\top})\\
&= (\begin{bmatrix}
 1 ~&  1
 \end{bmatrix} \cdot \begin{bmatrix}
 1 ~&  0
 \end{bmatrix}^{\top} ) \otimes  (  B(\sigma)_{(1 + i_1, \ldots, 1 + i_k)} \cdot  P^{j} \cdot  e_{1}^{\top})\\
&=  B(\sigma)_{(1 + i_1, \ldots, 1 + i_k)} \cdot ( e_{1}  P^{n-j})^{\top}\\
&=    B(\sigma)_{(1 + i_1, \ldots, 1 + i_k), (1+n-j) \bmod n}.
\end{align*}

Finally, we prove property \textit{(iii)}. If $\sum_{\sigma \in \Sigma\setminus \{\sigma_0, \sigma_1\}} \#_{\sigma}(t) \ge 2$ then there exists a subtree $\sigma^{\prime} (t_1, \ldots, t_k)$ of $t$ such that $k \ge 1$,  $\sigma^{\prime} \in \Sigma_k \setminus \{\sigma_1\}$, and $\sum_{\sigma \in \Sigma\setminus \{\sigma_0, \sigma_1\}} \#_{\sigma}(t_i) = 1$ for some $i \in [k]$. It follows from Equation (\ref{lower:propertiesofclass:part2}) that $\mu(t_i) = \begin{bmatrix}
 1 ~&  1
 \end{bmatrix} \otimes \alpha$ holds for some $\alpha \in \mathbb{F}^{1 \times n}$.
By the mixed-product property of Kronecker product and Equation (\ref{Kronecker:kmixed}), we have
\allowdisplaybreaks\begin{align*} 
\mu(\sigma^{\prime} (t_1, \ldots, t_k)) &=\left(  \bigotimes\limits_{j=1}^{k}  \mu(t_j) \right) \cdot \left(  \begin{bmatrix}
 1 ~&  1
 \end{bmatrix} \otimes \left(  \begin{bmatrix}
 I_{n} \\
-I_{n}
 \end{bmatrix}^{\otimes k} \cdot B(\sigma^{\prime})\right)\right)\\
 &=\begin{bmatrix}
 1 ~&  1
 \end{bmatrix} \otimes  \left(  \bigotimes\limits_{j=1}^{k}  \mu(t_j) \cdot   \begin{bmatrix}
 I_{n} \\
-I_{n}
 \end{bmatrix}^{\otimes k} \cdot B(\sigma^{\prime})\right)\\
 &=\begin{bmatrix}
 1 ~&  1
 \end{bmatrix} \otimes  \left( \bigotimes\limits_{j=1}^{k} \left( \mu(t_j)  \cdot   \begin{bmatrix}
 I_{n} \\
-I_{n}
 \end{bmatrix}\right) \cdot B(\sigma^{\prime})\right) = 0_{1 \times 2n}
\end{align*}
where the last equality holds since 
\allowdisplaybreaks\begin{align*} 
\mu(t_i)  \cdot   \begin{bmatrix}
 I_{n} \\
-I_{n}
 \end{bmatrix} =  \begin{bmatrix}
 \alpha ~&  \alpha
 \end{bmatrix}  \cdot   \begin{bmatrix}
 I_{n} \\
-I_{n}
 \end{bmatrix}  = 0_{1 \times n}.
\end{align*}
Since $\sigma^{\prime} (t_1, \ldots, t_k)$ is a subtree of $t$, we now have that $\mu(t) = 0_{1 \times 2n}$ and thus $f(t) = 0$. 
\end{proof}

\begin{remark} \label{lower:remark_modn}
As $P^{n} = I_{n}$, we have $\mu(\sigma_1)^{n} = I_{2 n}$. Thus for every $f \in \mathcal{C}$, $k \in \{0, 1, ..., m\}$, $\sigma \in \Sigma_k \setminus \{\sigma_0, \sigma_1\}$, and $j, i_1, \ldots, i_k \in \mathbb{N}_0$, it holds that $f(\sigma_{1}^{j} (\sigma_{0})) = f(\sigma_{1}^{j \bmod n} (\sigma_{0}))$ and
\begin{align*}
f (\sigma_{1}^{j} (\sigma(\sigma_{1}^{i_1} (\sigma_{0}), \ldots,\sigma_{1}^{i_k} (\sigma_{0})))) = f (\sigma_{1}^{j \bmod n} (\sigma(\sigma_{1}^{i_1 \bmod n} (\sigma_{0}), \ldots,\sigma_{1}^{i_k \bmod n} (\sigma_{0})))).
\end{align*}
\end{remark}

Run $\mathsf{L}$ on a target $f \in \mathcal{C}$. Lemma \ref{lower:propertiesofclass} (i), (iii) and Remark \ref{lower:remark_modn} imply that when $\mathsf{L}$ makes a membership query on $t \in T_{\Sigma}$ such that $\sum_{\sigma \in \Sigma\setminus \{\sigma_0, \sigma_1\}} \#_{\sigma}(t) \ge 2$, the Teacher returns $0$, while when $\mathsf{L}$ makes a membership query on $t = \sigma_{1}^{j} (\sigma_{0})$, the Teacher returns $1$ if $j \bmod n = 0$ and returns $0$ otherwise. In these cases, $\mathsf{L}$ does not gain any new information about $f$ since every function in $\mathcal{C}$ is consistent with the values returned by the Teacher. 

When $\mathsf{L}$ makes a membership query on a tree
$t = \sigma_{1}^{j} (\sigma(\sigma_{1}^{i_1} (\sigma_{0}), \ldots,\sigma_{1}^{i_k} (\sigma_{0})))$ 
such that $k \in \{0, 1, ..., m\}$ and $\sigma \in \Sigma_k \setminus \{\sigma_0, \sigma_1\}$, the Teacher returns 
an arbitrary number in $\mathbb{F}$ if the value $f(t)$ is not already known from an earlier query. Lemma \ref{lower:propertiesofclass} (ii) and Remark \ref{lower:remark_modn} imply that $\mathsf{L}$ thereby learns the entry $B(\sigma)_{(1 + (i_1 \bmod n), \ldots, 1 + (i_k \bmod n)), (1+n-j) \bmod n}$.

When $\mathsf{L}$ makes an equivalence query on a hypothesis $h \in \mathcal{C}$, the Teacher finds some entry $B(\sigma)_{(i_1, \ldots, i_k), j}$ that $\mathsf{L}$ does not know from previous queries and returns the tree $\sigma_{1}^{1+n-j} (\sigma(\sigma_{1}^{i_1 -1} (\sigma_{0}), \ldots,\sigma_{1}^{i_k -1} (\sigma_{0})))$ as the counterexample.
 
With each query, the Learner $\mathsf{L}$ learns at most one entry of $B(\sigma)$ where $\sigma \in \Sigma\setminus \{\sigma_0, \sigma_1\}$. The number of queries made by  $\mathsf{L}$ on target $f$ is, therefore, at least the total number of entries of $B(\sigma)$ for all $\sigma \in \Sigma\setminus \{\sigma_0, \sigma_1\}$. The latter number is equal to
\begin{align*}
\sum_{\sigma \in \Sigma\setminus \{\sigma_0, \sigma_1\}} n^{\mathit{rk}(\sigma) + 1} &\ge \frac{1}{2^{m +1}} \cdot \sum_{\sigma \in \Sigma\setminus \{\sigma_0, \sigma_1\}} r^{\mathit{rk}(\sigma) + 1} \\ &= \frac{1}{2^{m +1}} \cdot \left(\sum_{\sigma \in \Sigma} r^{\mathit{rk}(\sigma) + 1} - r^{2} - r \right). 
\end{align*}
\end{proof}

The lower bounds of Theorems \ref{anyfield} and \ref{specificfield} are both linear in the target automaton size. Note that when the alphabet rank is fixed, the lower bound for learning over a fixed field (Theorem \ref{specificfield}) is the same up to a constant factor as for learning over an arbitrary field (Theorem \ref{anyfield}). 

Assuming a Teacher that represents counterexamples as succinctly as
possible, e.g, using the algorithm of \citet*{seidlfull}, the upper
bound of algorithm $\mathsf{LMTA}$ from Theorem
\ref{MTAlearning_complexityanalysis} is quadratic in the target
automaton size, i.e., quadratically greater than the lower bound of
Theorem \ref{anyfield}.

\section{Future Work}
\citet{five} apply their exact learning algorithm for multiplicity
word automata to show exact learnability of certain classes of
polynomials over both finite and infinite fields.  They also prove
the learnability of disjoint DNF formulae (i.e., DNF formulae in
which each assignment satisfies at most one term) and, more generally,
disjoint unions of geometric boxes over finite domains. 

The learning framework considered in this paper involves tree
automata, which are more expressive than word automata.  Moreover, our
result on the complexity of equivalence of multiplicity tree automata
shows that, through equivalence queries, the Learner essentially has
an oracle for polynomial identity testing.  Thus a natural direction
for future work is to seek to apply our algorithm to derive new
results on exact learning of other concept classes, such as
propositional formulae and polynomials (both in the commutative
and noncommutative cases). In this direction, we would like to examine
the relationship of our work with that of \citet*{klivans2006learning}
on exact learning of algebraic branching programs and arithmetic
circuits and formulae.  The latter paper relies on rank bounds for
Hankel matrices of polynomials in noncommuting variables, obtained by
considering a generalised notion of partial derivative.  Here we would
like to determine whether the extra expressiveness of Hankel matrices
over tree series can be used to show learnability of more expressive
classes of formulae and circuits.

\citet{sakakibara1990learning} showed that context-free grammars
(CFGs) can be learned efficiently in the exact learning model using
membership queries and counterexamples based on parse trees. Given the
important role of weighted and probabilistic CFGs across a range of
applications including linguistics, another natural next step would be to
apply our algorithm to learn weighted CFGs similarly using queries and
counterexamples involving parse trees.
